\documentclass[accepted]{uai2021} 

\usepackage[american]{babel}
\usepackage{natbib} 
\usepackage{mathtools} 
\usepackage{booktabs} 
\usepackage{tikz} 
\usepackage{subfig}
\usepackage{comment}
\usepackage{mathtools}
\usepackage{epsfig}
\usepackage{graphicx}
\usepackage{amsmath}
\usepackage{amssymb}
\usepackage{amsthm}
\usepackage{thmtools, thm-restate}
\usepackage[utf8x]{inputenc} 
\usepackage{multirow}
\usepackage{tabularx}
\newcolumntype{Y}{>{\centering\arraybackslash}X}

\newtheorem{theorem}{Theorem}
\newtheorem{lemma}[theorem]{Lemma}
\title{RISAN: Robust Instance Specific Deep Abstention Network}

\usepackage{xr}
\makeatletter

\newcommand{\rrv}{\lambda}    
\newcommand{\rrvi}[1]{\lambda_{#1}} 
\newcommand{\urho}{\bar{\rho}}   
\newcommand{\pc}{p'} 
\newcommand{\Rhds}[2]{ \hat{R}_{ds} (#1, #2) }

\author[1]{\href{mailto:<bhavya.kalra@research.iiit.ac.in>?Subject=Your UAI 2021 paper}{Bhavya Kalra}{}}
\author[2]{\href{mailto:<t-kusha@microsoft.com>?Subject=Your UAI 2021 paper}{Kulin Shah}{}}
\author[1]{\href{mailto:<naresh.manwani@iiit.ac.in>?Subject=Your UAI 2021 paper}{Naresh Manwani}{}}
\affil[1]{%
    Machine Learning Lab\\
    International Institute of Technology\\
    Hyderabad, India
}
\affil[2]{%
    Microsoft Research\\
    Bangalore, India\\
}
  
\begin{document}
\maketitle

\begin{abstract}
In this paper, we propose deep architectures for learning instance specific abstain (reject option) binary classifiers.
    The proposed approach uses double sigmoid loss function as described by Kulin Shah and Naresh Manwani in ("Online Active Learning of Reject Option Classifiers", AAAI, 2020), as a performance measure. We show that the double sigmoid loss is classification calibrated. We also show that the excess risk of 0-d-1 loss is upper bounded by the excess risk of double sigmoid loss. We derive the generalization error bounds for the proposed architecture for reject option classifiers. To show the effectiveness of the proposed approach, we experiment with several real world datasets. We observe that the proposed approach not only performs comparable to the state-of-the-art approaches, it is also robust against label noise. We also provide visualizations to observe the important features learned by the network corresponding to the abstaining decision.
\end{abstract}

\section{Introduction}
In classification problems, learning becomes difficult when the cost of misclassification is extremely high. It becomes more challenging when learning critical tasks such as stock markets, medical diagnosis, autonomous driving, biotech, cyber-security, identification technologies, and robot-assisted surgery. In such situations, it becomes advantageous to refrain from taking any decision when in a dilemma. Such classifiers are called abstain (reject) option classifiers. Abstain classifiers have been successfully used in medical diagnosis \citep{10.1007/978-3-642-21257-4_73}, financial forecasting \citep{Rosowsky2013RejectionBS}, genomics \citep{PMID:18621758}, speech emotion recognition \citep{Sridhar2019}, crowdsourcing \citep{7747496} etc.

Let $\mathcal{X}\subset \mathbb{R}^D$ be the feature space and $\{+1,-1\}$ be the label space. An abstaining classifier can be defined using a function $f:\mathcal{X}\rightarrow \mathbb{R}$ and a rejection function $\rho: \mathcal{X}\rightarrow  \mathbb{R}_+$ as follows. 
\begin{align*}
    g(f(\mathbf{x}),\rho(\mathbf{x}))=\begin{cases}
    1, & \mathbb{I}[f(\mathbf{x})>\rho(\mathbf{x})]\\
    \text{reject}, & \mathbb{I}[\vert f(\mathbf{x})\vert \leq \rho(\mathbf{x})]\\
    -1, & \mathbb{I}[f(\mathbf{x})<-\rho(\mathbf{x})]
    \end{cases}
\end{align*}
The goal here is to simultaneously learn the function $f(\cdot)$ and $\rho(.)$. The performance of a given abstain classifier is measured using loss $L_d$ (0-d-1) as follows.
\begin{align}
   \label{eq:0-d-1-loss}
   L_d(yf(\mathbf{x}),\rho(\mathbf{x}))=\mathbb{I}[yf(\mathbf{x})&< -\rho(\mathbf{x})] + d\;\mathbb{I}[\vert yf(\mathbf{x})\vert \leq \rho(\mathbf{x})]
\end{align}
where $d\in (0,0.5)$ is the cost of rejection.
Loss $L_d$ is minimized by {\em generalized Bayes classifier} \citep{chow1970optimum} described as follows. 
\begin{align}
    \label{eq:generalized-bayes-classifier}
    f_d^*(\mathbf{x})=\begin{cases}
    1, & \eta(\mathbf{x})>1-d \\
    \text{reject},  & d\leq \eta(\mathbf{x}) \leq 1-d\\
    -1 , & \eta(\mathbf{x}) <d
    \end{cases}
\end{align}
where $\eta(\mathbf{x})=P(y=1|\mathbf{x})$. Loss $L_d$ is discontinuous. Thus, minimizing risk under $L_d$ is difficult. In practice, various surrogate losses of $L_d$ have been used for learning abstain classifiers.  

\paragraph{Kernel Based Approaches: }Different algorithms for learning abstaining classifiers are proposed based on different choices of surrogates of $L_d$. Generalized hinge \citep{bartlett2008classification} and double hinge \citep{grandvalet2009support} are convex surrogates of $L_d$. Risk minimization using these losses results in support vector machine (SVM) like algorithms. However, approaches proposed in \citep{bartlett2008classification,grandvalet2009support} learn the rejection bandwidth as a post-processing step resulting in suboptimal solutions. \citet{manwani2015double, shah2019sparse} propose approaches based on nonconvex surrogate of $L_d$ called double ramp loss. \citet{cortes2016learning} propose max-hinge loss and plus-hinge loss for rejection option and propose a kernel-based approach that minimizes these losses. Online active learning of abstaining classifiers is discussed in \citep{shah2020online}. These approaches face three major challenges. (a) These approaches rely on kernel trick to learn nonlinear classifiers. Thus, the scalability of these methods with big data is an issue. (b) Function $\rho(.)$ is assumed to be a constant for all instances (i.e., $\rho(\mathbf{x})=\rho,\;\forall \mathbf{x}\in \mathcal{X}$). Thus, these approaches do not learn instance-specific rejection functions. (c) Most of these approaches are not robust against the label noise. Though the approach proposed in \cite{shah2019sparse} is shown robust against label noise, it uses kernels to learn nonlinear classifiers and cannot produce instance-specific rejection bandwidth. 

\paragraph{Deep Learning-based Approaches for Abstain Classifiers: } A neural networks based classifiers with abstain option is proposed  \citet{de2000reject}. In this model, rejections are done after the learning of the classifier. This results in a suboptimal abstain option classifier. A similar approach for deep neural networks(DNNs) is proposed in \citet{geifman2017selective}, which finds the best abstaining threshold based on the softmax output corresponding to each class from already trained networks. The method proposed in \citet{el2010foundations} optimizes a pair of functions, a classification function, and a selective function with a risk-coverage trade-off, where coverage is defined as the ratio of samples selected for classification amongst the complete dataset. Deep learning implementation of the same is proposed in Selectivenet \citep{geifman2019selectivenet}. This approach learns the appropriate selection and classification function for a given coverage in a deep learning setting. However, this approach does not take rejection cost $d$ into account in their objective function. The main issue with such an approach is that it does not allow the data to decide the rejection rate. For example, in instances where the classes are separable with sufficient margin, this approach rejects and learns the classifier using the remaining examples based on specified coverage parameters. \citet{thulasidasan2018knows} consider abstaining option as another class. However, this changes the abstain option's interpretation as the purpose of abstaining option is to capture the overlapping regions of any two classes.

\paragraph{Proposed Approach: }In this paper, we propose an instance-specific deep learning approach with abstain option. The proposed approach takes the cost of rejection also as an input. It simultaneously learns the decision surface ($f(\mathbf{x})$) and rejection function ($\rho(\mathbf{x})$) which depends on the cost of rejection $d$. We use double sigmoid loss function to compare the output of the network with the ground truth. Note that the double sigmoid loss is a smooth nonconvex surrogate of $L_d$ (see Eq.~(\ref{eq:0-d-1-loss})).

\paragraph{Key Contributions: }Our key contributions in this paper are as follows.
\begin{enumerate}
\item We show that the double sigmoid loss function is classification calibrated. We provide the excess risk bounds of the double sigmoid loss. 
\item We propose a novel instance-specific deep abstain network called RISAN. RISAN has two variants, with and without instance-specific rejection function. 
\item We derive the generalization error bounds for the proposed approach RISAN. 
\item We show the proposed approach's effectiveness by comparing it with various state-of-the-art algorithms on various benchmark datasets. We also show by experiments that RISAN is robust against label noise in the data.
\item We also show visualizations that focus on the areas in an image leading the network to choose to abstain option. These visualizations reflect that our network learns useful representations for the rejection as well as classification.
\end{enumerate}

\paragraph{Paper Organization: }The rest of the paper is organized as follows. We discuss the double sigmoid loss and its properties in Section~2. In Section~3, we discuss the proposed approach RISAN, its different variants, and generalization bounds. We show the experimental results in Section~4. Robustness results of RISAN are given in Section~5. We discuss the visualizations of the representations learned by RISAN in Section~6. We conclude the paper with some remarks and future directions in Section~7. 
\section{Double Sigmoid Loss for Abstention}
As discussed earlier, $\mathcal{X}\subseteq \mathbb{R}^D$ is the feature space and $\mathcal{Y}\in\{\pm 1\}$ is the label space. Let $\mathcal{P}(\mathbf{x},y)$ be the unknown joint distribution on $\mathcal{X}\times \mathcal{Y}$. Let $\mathcal{S}=\{(\mathbf{x}_1,y_1),\ldots,(\mathbf{x}_N,y_N)\}$ be the finite training set where each $(\mathbf{x}_i,y_i)$ is generated i.i.d. from the distribution $\mathcal{P}(\mathbf{x},y)$. The goal here is to learn functions $f: \mathcal{X} \rightarrow \mathbb{R}$ and $\rho:\mathcal{X}\rightarrow \mathbb{R}_+$ using the training set $\mathcal{S}$. 

Here, functions $f(.)$ and $\rho(.)$ are represented using deep neural network (to be discussed shortly). To evaluate the performance of the learnt functions $f(.)$ and $\rho(.)$, we use double sigmoid loss function \cite{shah2020online} as follows.
\begin{align}
    L_{ds}(yf(\mathbf{x}), \rho(\mathbf{x})) &= 2d\sigma(yf(\mathbf{x}) - \rho(\mathbf{x})) \nonumber \\ &\quad+ 2(1 - d)\sigma(yf(\mathbf{x}) + \rho(\mathbf{x})) 
    \label{eq:loss_fn}
\end{align}
where $d$ is the cost of rejection
and $\sigma(a) = (1 + \exp{(\gamma a)} )^{-1}$ is the sigmoid function with $(\gamma > 0)$. The risk under double sigmoid loss function is as follows. 
\begin{equation}
   R_{ds}(f,\rho) = \mathbb{E}_{\mathcal{X},\mathcal{Y}}
   \left[L_{ds}(yf(\mathbf{x}), \rho(\mathbf{x}))\right] \nonumber
\end{equation}
Here, we establish theoretical properties of the double sigmoid loss function.
\paragraph{Classification Calibration}
Double sigmoid loss is a linear combination of two sigmoid functions and hence is a non convex loss function. We first show classification calibration on double sigmoid loss by ensuring that the risk under $L_{ds}$ is minimized by the generalized bayes classifier. To approximate the optimal classifier, classification calibration is the minimal requirement for any loss function. 

\begin{theorem}
\label{THM:CLASSCALIB}
For a fixed cost of rejection $d$, the risk under double sigmoid loss is minimized by the generalized Bayes classifier $f_d^*(.)$ (see Eq.(\ref{eq:generalized-bayes-classifier})).
\end{theorem}
The proof of the theorem is provided in \ref{Appendix:A1}
\paragraph{Excess Risk Bound}
We now relate the excess risk of $L_{d}$, ($R_{d}(f,\rho)-R_{d}(f_d^*)$) with the excess risk of the double sigmoid loss ($R_{ds}(f,\rho)-R_{ds}(f_d^*)$). Note that here $R_{d}(f,\rho)=\mathbb{E}_{\mathcal{X},\mathcal{Y}}[L_{d}(yf(\mathbf{x}),\rho(\mathbf{x}))]$ and $f_d^*$ (see Eq.(\ref{eq:generalized-bayes-classifier})) is the generalized Bayes classifier which minimizes $R_{d}(f,\rho)$. $R_d(f_d^*)$ and $R_{ds}(f_d^*)$ represents risk of generalized Bayes classifies $f_d^*$ under $L_{d}$ and $L_{ds}$ loss. We know that $L_{d}(yf(\mathbf{x}),\rho(\mathbf{x}))\leq L_{ds}(yf(\mathbf{x}),\rho(\mathbf{x}))$. Thus, taking expectations on both sides, we get, $R_{d}(f,\rho)\leq R_{ds}(f,\rho)$.
We follow the approach of \cite{bartlett2006convexity} to establish an excess risk bound for the double sigmoid loss function $L_{ds}$. 

\begin{theorem}
\label{THM:EXCESS}
Let $0\leq d \leq 1/2$ and a measurable function $z$. Then we have the excess risk relation as 
\begin{equation}
     \psi\left(R_{d}(f,\rho)-R_{d}(f_d^*)\right) \leq R_{ds}(f,\rho) - R_{ds}(f_d^*) \nonumber
\end{equation}
where  
\begin{equation}
    \psi(\theta) = \begin{cases}
    0 & \theta = 0 \\
    (2d-1)\zeta+\left(\frac{\theta+1-2d}{2}\right)\left(\frac{T+\zeta^{2}\theta}{\zeta\theta+T\zeta}\right)\\\hspace{1.25 cm}+\left(\frac{\theta+2d-1}{2}\right)\left(\frac{T-\zeta^{2}\theta}{\zeta\theta-T\zeta}\right) & \theta \in (0,1-2d] \\
    \theta + (2d-1)\zeta & \theta \in [1-2d,1]
    \end{cases} \nonumber
\end{equation}
and $\theta = R_{d}(f,\rho) - R_{d}(f_{d}^{*})$ and . Also, $\zeta=tanh(\frac{\rho}{2})$ and $T=(1-2d)-\sqrt{(1-2d)^{2}-\theta^{2}}$. 
\end{theorem}
The proof of the theorem is provided in  \ref{Appendix:A2}

Since we have established statistical properties of double sigmoid loss, we can use this loss in deep networks to train classifiers with abstention option.
\begin{figure}[h]
  \centering
     \includegraphics[width=0.35\textwidth]{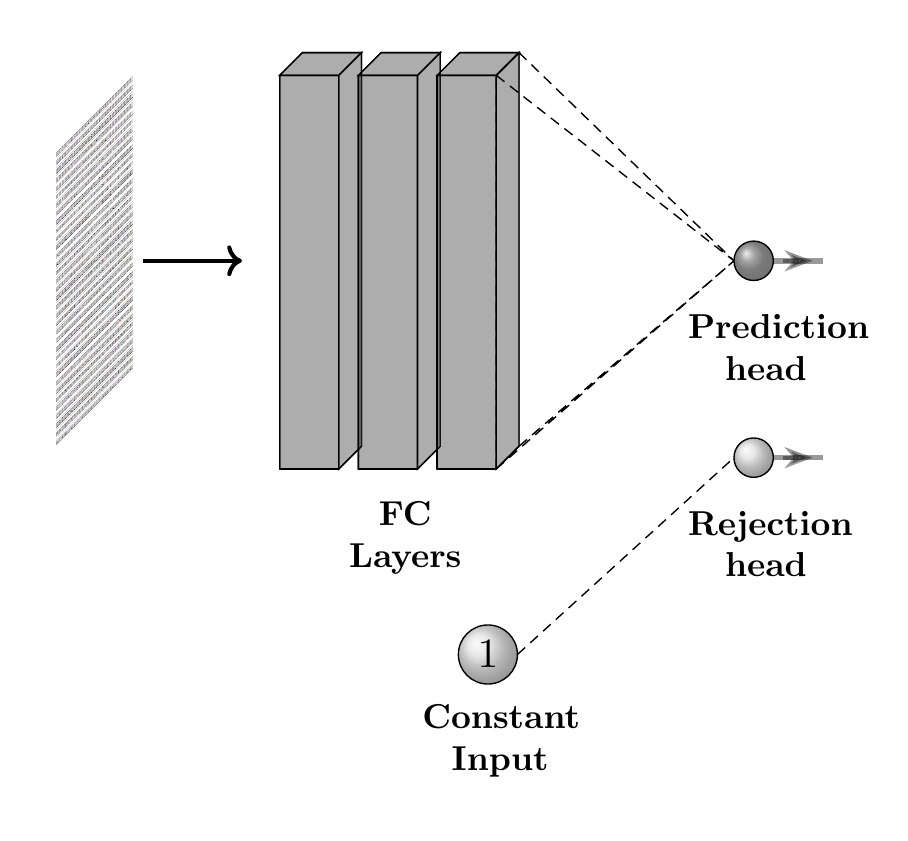}
  \caption{RISAN architecture with input independent $\rho$}
    \label{fig:DNN-1}
  \end{figure}
  
\begin{figure*}[t]
\centering
    \subfloat[RISAN without an auxiliary head (RISAN-NA) \label{fig:DNN-na-1}]{
        \includegraphics[width=0.45 \textwidth]{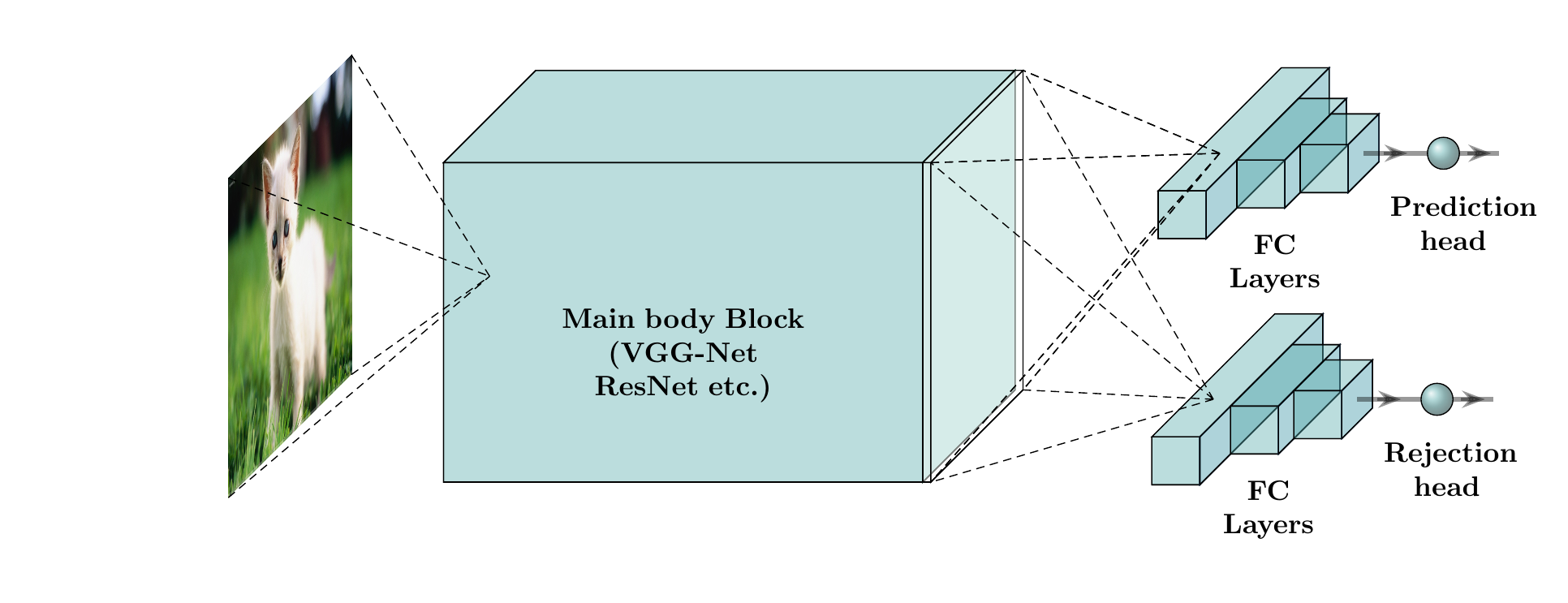}
        %
     }
  \subfloat[RISAN with an auxiliary head (RISAN) \label{fig:DNN-2}]{
        \includegraphics[width=0.45 \textwidth]{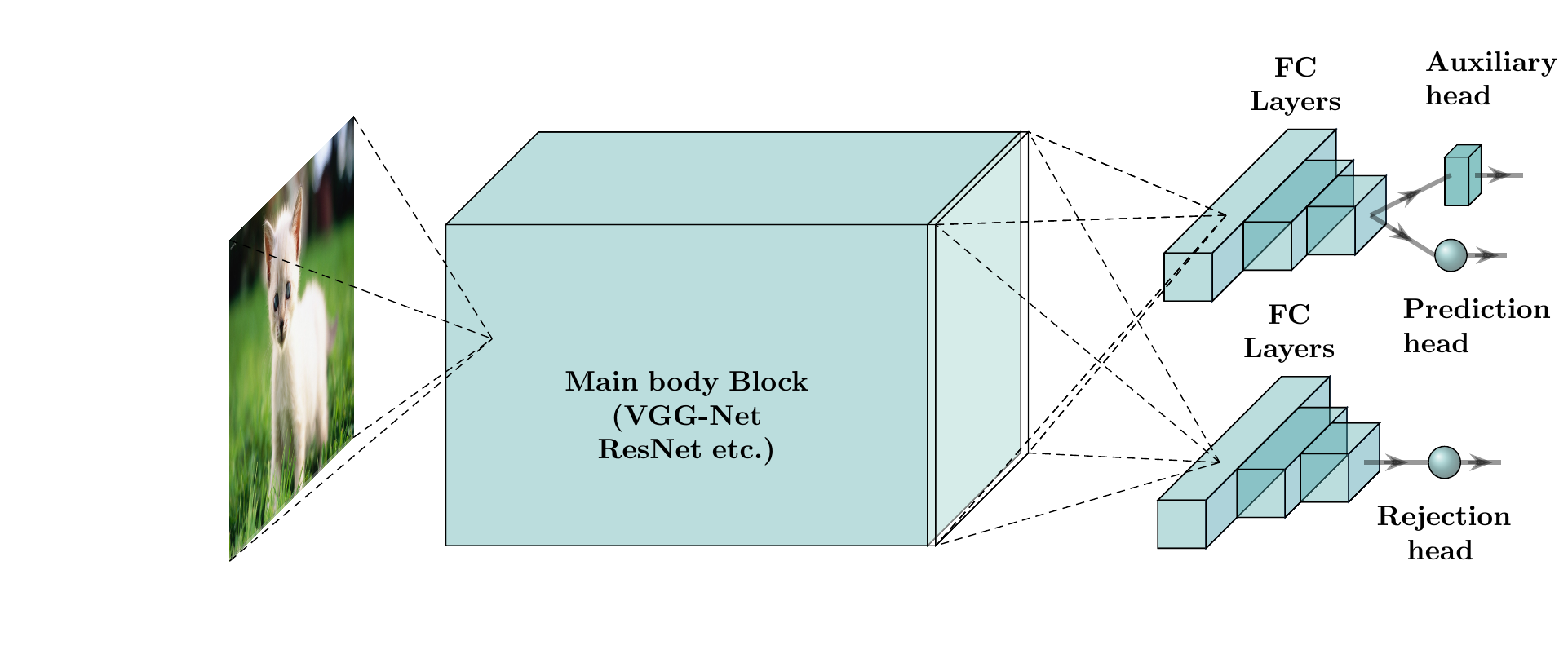}
  }
  \caption{Different implementations of RISAN with input dependent $\rho(.)$}
  \label{fig:inp-dep}
\end{figure*}

\section{Proposed Approach: RISAN}

The proposed architecture models both decision surface $f(.)$ and rejection function $\rho(.)$ in a single DNN model. Schematic view of RISAN implementations is depicted in Figure-\ref{fig:DNN-1} and Figure-\ref{fig:inp-dep}. 
The network's input is processed by the main body block and an associated (separate or same) network that would learn the rejection region parameter. The main body block consists of hidden layers or sub-blocks. The rejection function $\rho(.)$ can be modeled by a separate single neuron or a network similar to the main body block. The main body block can be assembled using any type of architecture relevant to the problem at hand (e.g., convolutional, fully connected, or recurrent architectures). 
\subsection{RISAN: Input Independent Rejection}
RISAN architecture represented in Figure-\ref{fig:DNN-1} describes the architecture when the rejection function takes the same value for all $\mathbf{x}$, that is, $\rho(\mathbf{x})=\rho,\;\forall \mathbf{x} \in \mathcal{X}$. RISAN for input independent rejection has two output heads, prediction head ($f(\mathbf{x})$) and rejection region parameter($\rho$). The input data $\mathbf{x}$ is fed into the fully connected (FC) layers while a fixed constant is fed into the rejection head. 
The role of the prediction head is to learn the appropriate decision surface $f(\mathbf{x})$, and the rejection head learns the rejection region parameter (denoted as $\rho$). In this case, the main body block is a stack of fully connected layers that are used for processing the input data.
\subsection{RISAN: Input-dependent Rejection}
 RISAN architecture in Figure-\ref{fig:inp-dep} describes the architectures when rejection function depends on the specific instance.  The primary architecture is provided in Figure~\ref{fig:DNN-na-1}
 for input dependent rejection.
 This architecture has two output heads similar to the input independent architecture. However, the rejection head is fed the input from the main body block. An additional architecture for incorporating auxiliary loss has been provided in Figure~\ref{fig:DNN-2}. This architecture
 has three output heads, prediction head ($f(\mathbf{x})$), rejection head ($\rho(\mathbf{x})$) and an auxiliary head.
 The auxiliary head, only used for training the networks, sometimes plays an important role in the initial process of acquiring complex features from convolutional blocks. We follow the notion of the auxiliary head for very deep neural networks as mentioned in \cite{geifman2019selectivenet}. The auxiliary head's role is to learn a related prediction task that facilitates the consolidation of apropos features in the main body block. 
 Thus, the prediction and rejection head are optimized with the auxiliary head helping build features that minimize $L_{c}$, the convex combination of categorical cross entropy loss $L_{ce}$ and double sigmoid loss $L_{ds}$.
 \begin{equation*}
    L_{c} = \alpha\times L_{ds} + (1-\alpha)\times L_{ce}    
\end{equation*}
The number and size of fully connected layers preceding these two or three heads (depending on the architecture) are independent and can vary depending on the task type and complexity. 
The final neuron, however, for both the prediction head and rejection head are single neurons. The final layer of auxiliary head $h(\mathbf{x})$ depends on the application and could be a softmax layer. 
The relevance of the different architectures has been explored in the experiments section.



\subsection{Generalization Error Bounds of RISAN with Input Independent Rejection}

 We followed the approach of \cite{neyshabur2015norm} to establish an upper bound on the Rademacher complexity of regularized DNN with double sigmoid loss function and an input independent $\rho$ as shown in figure \ref{fig:DNN-1}. We show in Theorem~\ref{THM:GENBOUND} that the Rademacher complexity for rectified linear unit based neural networks and consider two intuitive types of norm regularization (i) bounding the norm of the incoming weights of each unit (per-unit regularization) and (ii) bounding the overall norm of all the weights in the system jointly (overall regularization)
Let $\ell_{p}$, be the norm over all incoming weights to each unit and $\ell_{q}$, the norm over all the units collectively. Now, considering the above definitions. Our neural network can be defined as a graph with group norm regularization as:
\begin{equation}
    \xi_{p, q}(\mathbf{w})=\left(\sum_{v \in V}\left(\sum_{(u \rightarrow v) \in E}|\mathbf{w}(u \rightarrow v)|^{p}\right)^{q / p}\right)^{1 / q} \nonumber
\end{equation}
where $u$ and $v$ are nodes in adjacent layers belonging to set of vertices, $V$. And $\mathbf{w}(u \rightarrow v)$ represents the weight associated with the edge $u\rightarrow v$ belonging to set of edges, $E$.

Let us consider a deep abstain network with $n+1$ layers including input and output layers. Let us assume that all the hidden layer have the same number of nodes $(H)$. Let $W_j$ denotes the weight matrix corresponding to the connections from $(j-1)$ layer to $j^{th}$ layer. Then, $W_{1}\in\mathbb{R}^{H\times D}$, $W_{2}$,\ldots,$W_{n-1}\in\mathbb{R}^{H\times H}$, $W_n\in \mathbb{R}^{1\times H}$ and $\rho(\mathbf{x})=\rho$. The output of the network can be defined written as, 
\begin{equation}
    f_{W}(x)=W_n \sigma\left(W_{n-1} \sigma\left(W_{n-2}\left(\ldots \sigma\left(W_{1} \mathbf{x}\right)\right)\right)\right)-\rho \nonumber
\end{equation}
where $\sigma$ is the activation function and $W=(W_1,\ldots,W_n,\rho)$. Let $\mathcal{F}$ denotes the set all possible functions represented by such a neural network. 
\begin{theorem}
\label{THM:GENBOUND}
Let $\mathcal{D}$ be any distribution on $\mathcal{X}\times \{-1,+1\}$. Let $0<\delta\leq 1$. Then for any $n$, $q\geq1$, $1\leq p <\infty$ and any set $S=\{\mathbf{x}_{1},\ldots,\mathbf{x}_{m}\}$; with probability at least $1-\delta$ (over $S\sim \mathcal{D}^m$), all functions $f\in \mathcal{F}$ satisfy

\begin{align}
R_{ds} (f, \rho) &\leq \Rhds{f, \rho}
+\frac{\urho}{\sqrt{m}}+ \sqrt{\frac{8\ln \left(\frac{4}{\delta}\right)}{m}} + \sqrt{\frac{2\ln \left(\frac{2}{\delta}\right)}{m}} \nonumber \\&\hspace{-0.8 cm}+
\left(\frac{2\beta}{\sqrt{m}}\max_{i}\|\mathbf{x}_{i}\|_{\pc}\right)\left(2H^{\left[\frac{1}{\pc}-\frac{1}{q}\right]_{+}}\right)^{n-1} \nonumber
\end{align}
where $n$ is the number of layers in the network, $H$ is the number of neurons in the hidden layers, rejection region parameter is bounded as $\rho\leq\urho$. Also $\frac{1}{\pc}+\frac{1}{p}=1$ and $[a]_+=\max(0,a)$. $\operatorname{er}_{S}^{\ell}\left[f\right]$ is the empirical error and $\beta_{p,q}(W) = \prod_{k=1}^{n}\|W_{k}\|_{p,q}\leq \beta$.
\end{theorem}
The proof of the theorem is provided in \ref{Appendix:A3}.
The key observations from the bound in Theorem~3 are as follows. The bounds depend on the number of neurons in each layer, $H$ and the number of layers $n$. The bounds are also inversely proportional to the number of samples, $\sqrt{m}$. Thus, increasing $m$ decreases the generalization error bound. Also, when $\pc\geq q$, the dependence on the number of neurons in each layer vanishes. If we use overall $\ell_{1}$ or $\ell_{2}$ regularization, this dependence should disappear.  

\subsection{Example: Classifier Learnt Using RISAN}
We generated 1000 examples in the square $[-1.5,1.5]^2$ uniformly randomly. We used $x_2-x_1-2sin(x_1)=0$ as separation boundary. We ensured equal representation of each class. We then randomly flipped labels of the samples present within the $\pm0.75$ margin of the decision boundary. We also used RISAN with input independent $\rho$ (see Figure~\ref{fig:DNN-1}) and $d=0.25$. The resulting classification boundary and rejection region of the synthetic dataset are shown in Figure~\ref{fig:2D classifier} where the dark region signifies the rejection region.
\begin{figure}[h]
\centering
    \subfloat[ \label{fig:2D classifier}]{\includegraphics[width=0.23\textwidth]{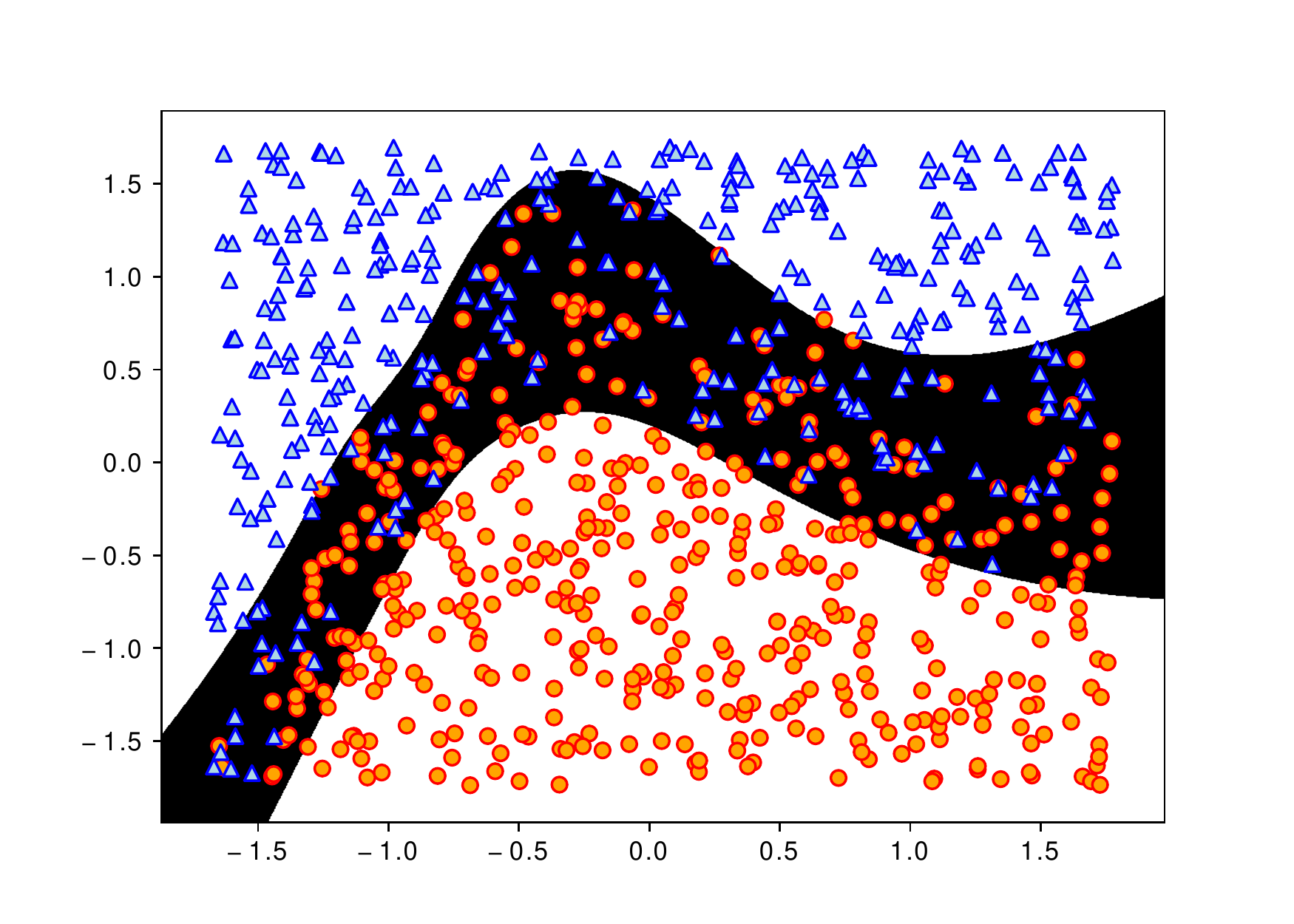}}
    \subfloat[\label{fig:Generalization Error}]{ 
    \includegraphics[width=0.23\textwidth]{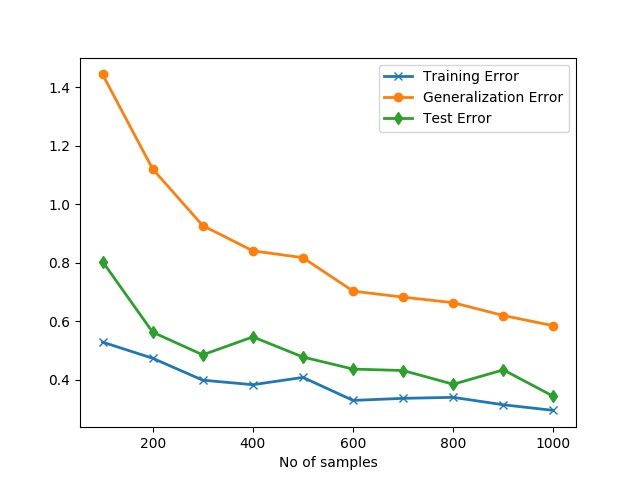}
    }
    \caption{ (a) Resulting classifier on a synthetically generated dataset with a nonlinear rejection region (black) (b) Generalization error upper bounds the Test error on the synthetic dataset}
\end{figure}
We also plot the generalization bounds for the input independent rejection on a 2D dataset (see Figure~\ref{fig:Generalization Error}). For $d=0.25$, we ran the experiments for $30$ epochs, increasing the no. of samples from $100$ to $1000$ with a step size of $100$. We observed that an increase in the number of samples leads to decreased training error, test error, and generalization error simultaneously. Also, we observed that the generalization error upper bounds the test error for each experiment. 

\section{Experiments}

This section describes the experimental details: datasets used, baseline algorithms used for comparison purposes, and our choice of architectures and hyper-parameters.

\subsection{Datasets Used}
Note that our proposed approach works for binary classification problems. Thus, to show the effectiveness of the proposed approach, we performed experiments on the following datasets. 
\begin{enumerate}
    \item Small Datasets: Ionosphere and ILPD \citep{Dua:2019}. 
    \item Phishing dataset \citep{Dua:2019}.
    \item Cats vs. Dogs \citep{elson2007asirra}: Each image re-scaled to 64x64 from original images of size 360x400.
    \item CIFAR-10 \citep{krizhevsky2009learning}: We selected classes \textit{automotive} and \textit{truck} from CIFAR-10 for our task. We have selected these classes as they have many similarities, contain overlapping features, and are tough to classify even for humans sometimes.
    \item MNIST \citep{lecun2010mnist}: We selected classes \textit{1} and \textit{7} from MNIST dataset for our task.
    \item CBIS-DDSM \citep{lee2017curated}: This is a medical image dataset with positive referring to the presence of some form of calcification or mass, and the absence refers to negative examples. The dataset has 14\% positives, and 86\% negative labeled pre-processed images with ROI extracted. We further sampled the images to create a subset dataset with a similar number (4500) of positives and negative examples each, all re-scaled to 64x64 from the original size of 299x299. 
\end{enumerate}
We divide our experiments into two categories, namely, small dataset and large dataset experiments because some baseline methods are optimized for the smaller datasets and fail to converge for larger datasets and. Hence, we use different baseline methods for small and large datasets.

\subsection{Baselines}
\paragraph{Baselines for Small Datasets Experiments:}
We compare our network with two state of the art methods, (a) DH-SVM: reject option classifier introduced in \citet{grandvalet2009support} which minimizes the double hinge loss and (b) SDR-SVM: sparse reject option classifier proposed in \citet{shah2019sparse} which minimizes $\ell_{1}$ regularized risk under double ramp loss function. 

\paragraph{Baselines for Large Datasets Experiments:}
We compare the proposed approach with the following baselines for Cats vs. Dogs, CIFAR-10, CBIS-DDSM, MNIST, and Phishing website datasets. (a) SelectiveNet(SNN) \citep{geifman2019selectivenet}: a deep neural architecture with an integrated reject option that simultaneously optimizes a prediction and a selection function . 
We also compare results on a variant of SNN without the auxiliary loss, the SNN-NA.
(b) DAC: deep abstaining classifier, a deep neural network trained with a modified cross entropy loss function introduced in \citet{thulasidasan2018knows} to accommodate an abstain (reject) class.

\begin{figure}[h]
  \centering
    \subfloat[Ionosphere risk]{\includegraphics[width=0.22\textwidth]{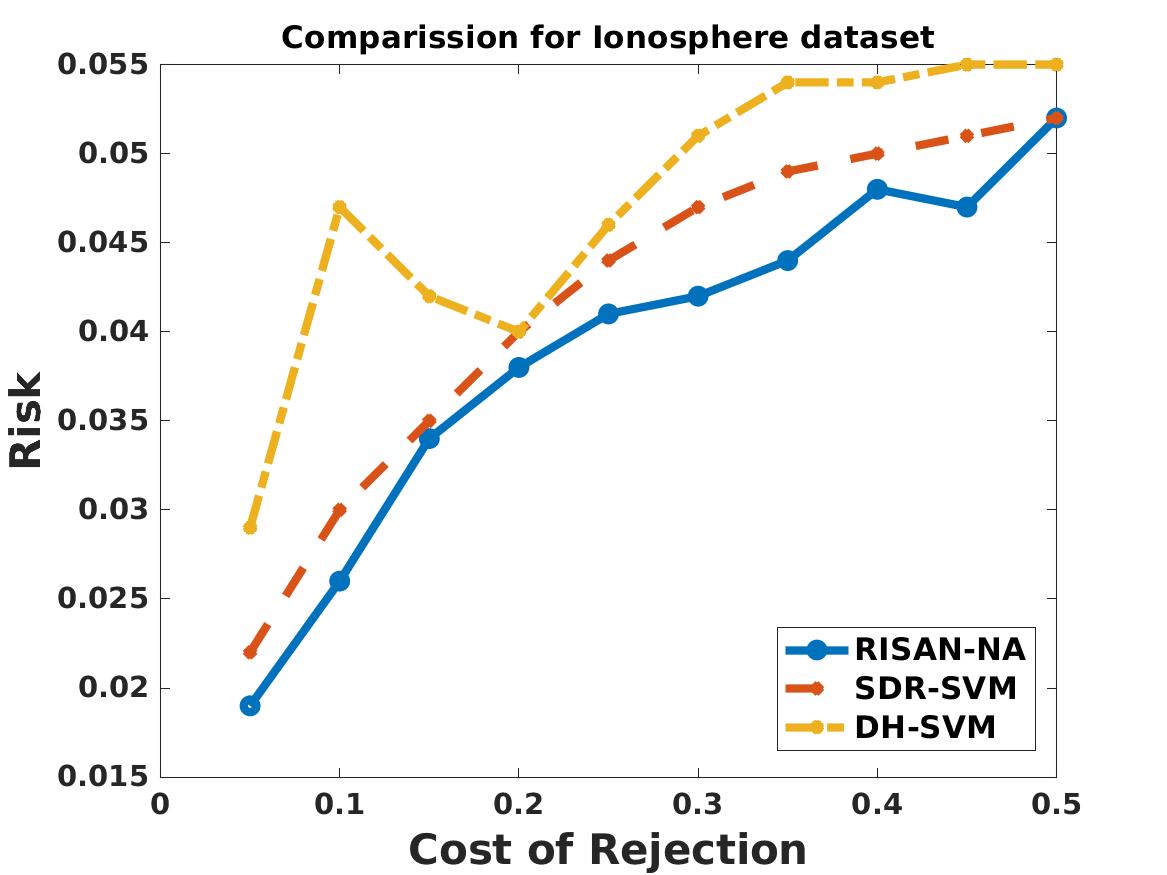}
  \label{fig:iono-risk}}
    \subfloat[ILPD risk]{\includegraphics[width=0.22\textwidth]{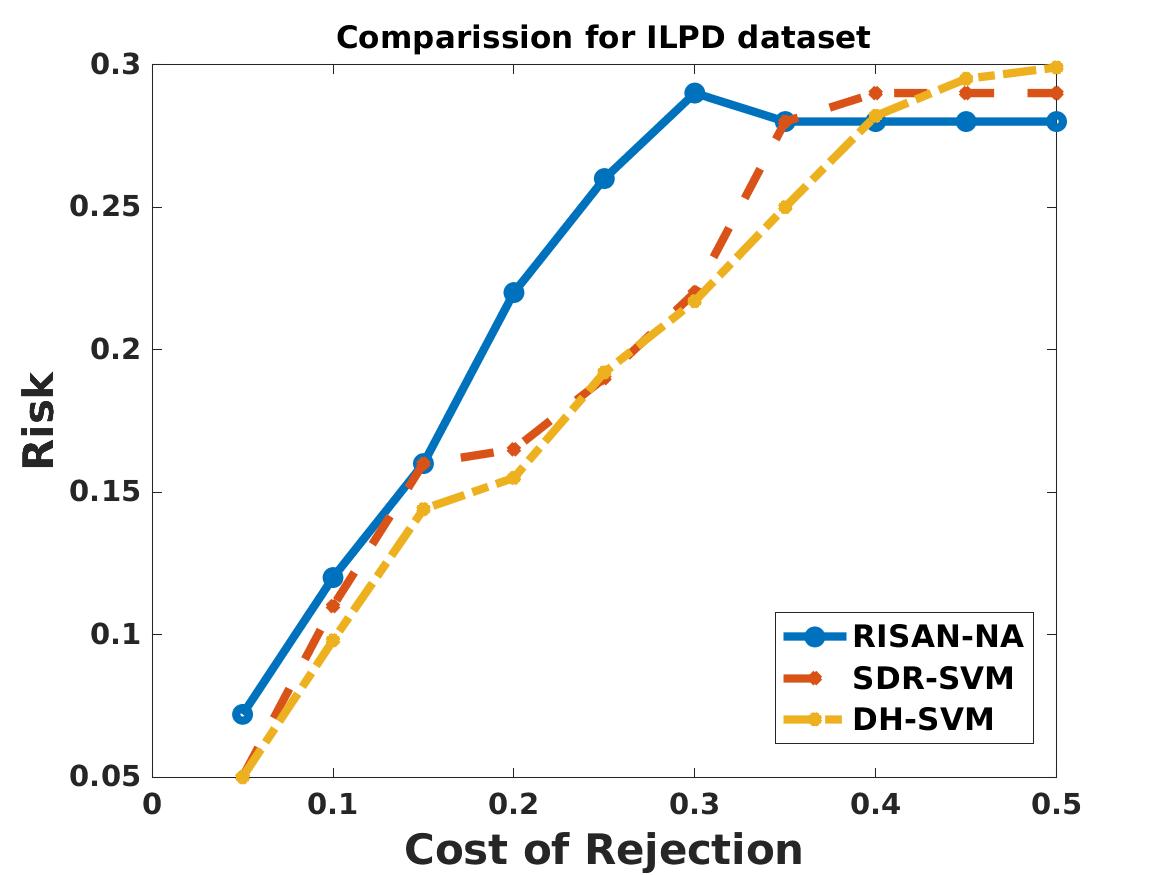}
  \label{fig:ilpd-risk}}
   \hfill
  \subfloat[Ionosphere accuracy]{\includegraphics[width=0.22\textwidth]{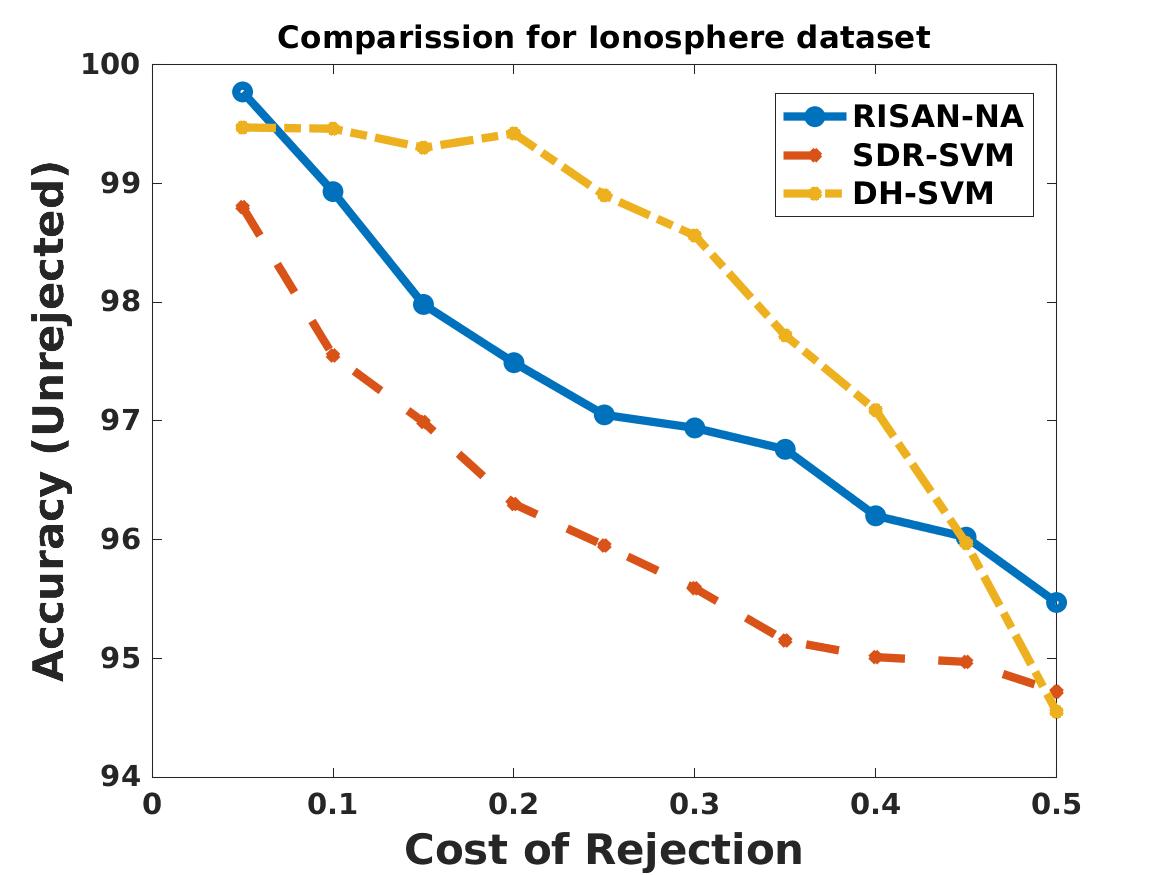}
  \label{fig:iono-accuracy}}
  \subfloat[ILPD accuracy]{\includegraphics[width=0.22\textwidth]{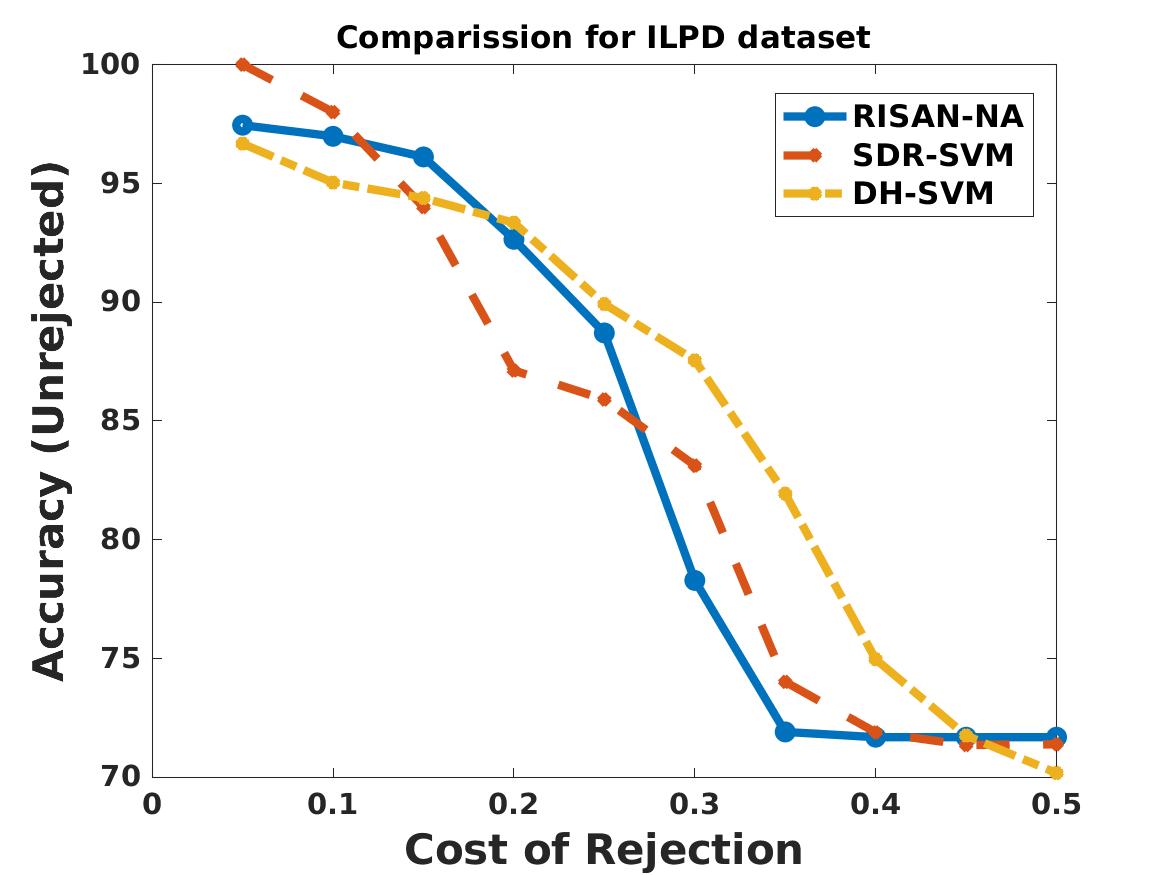}
  \label{fig:ilpd-accuracy}}
 
  \subfloat[Ionosphere rejection rate]{\includegraphics[width=0.22\textwidth]{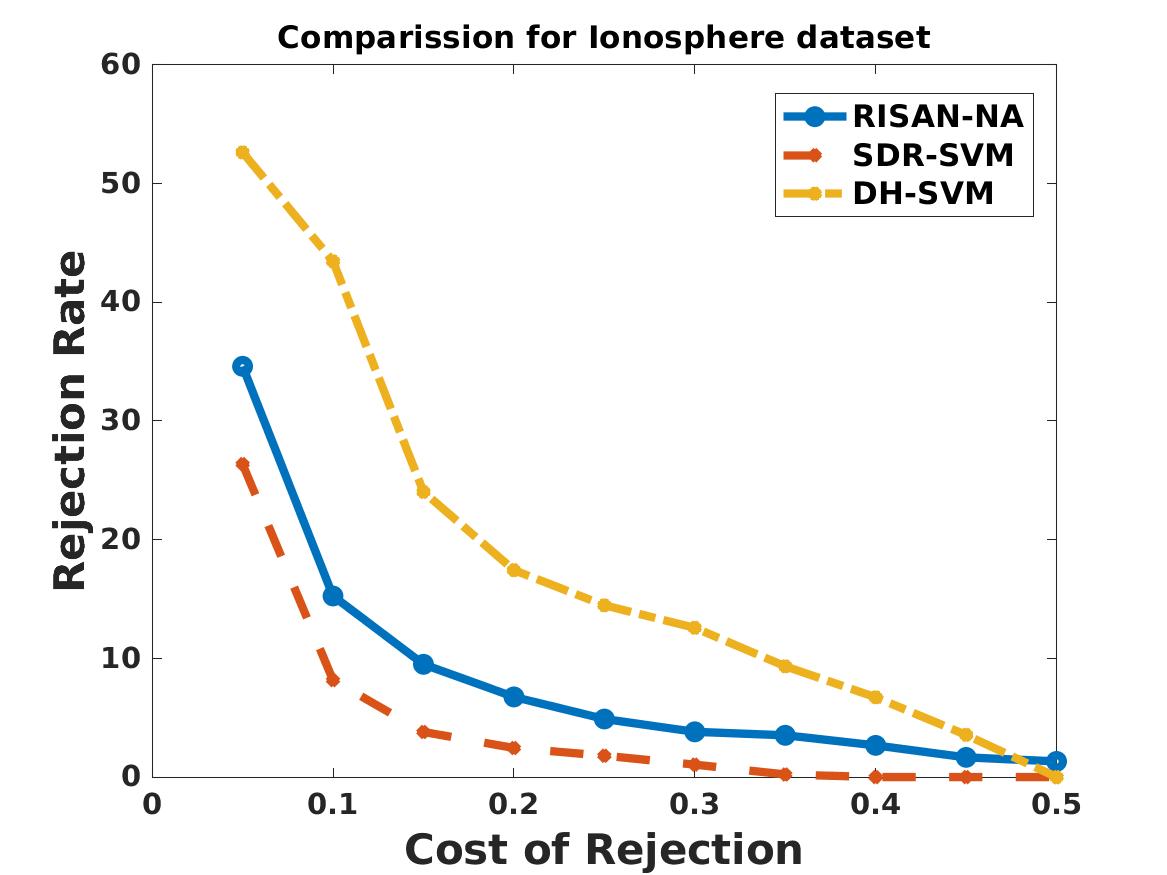}
  \label{fig:iono-rejection}}
  \subfloat[ILPD rejection rate]{\includegraphics[width=0.22\textwidth]{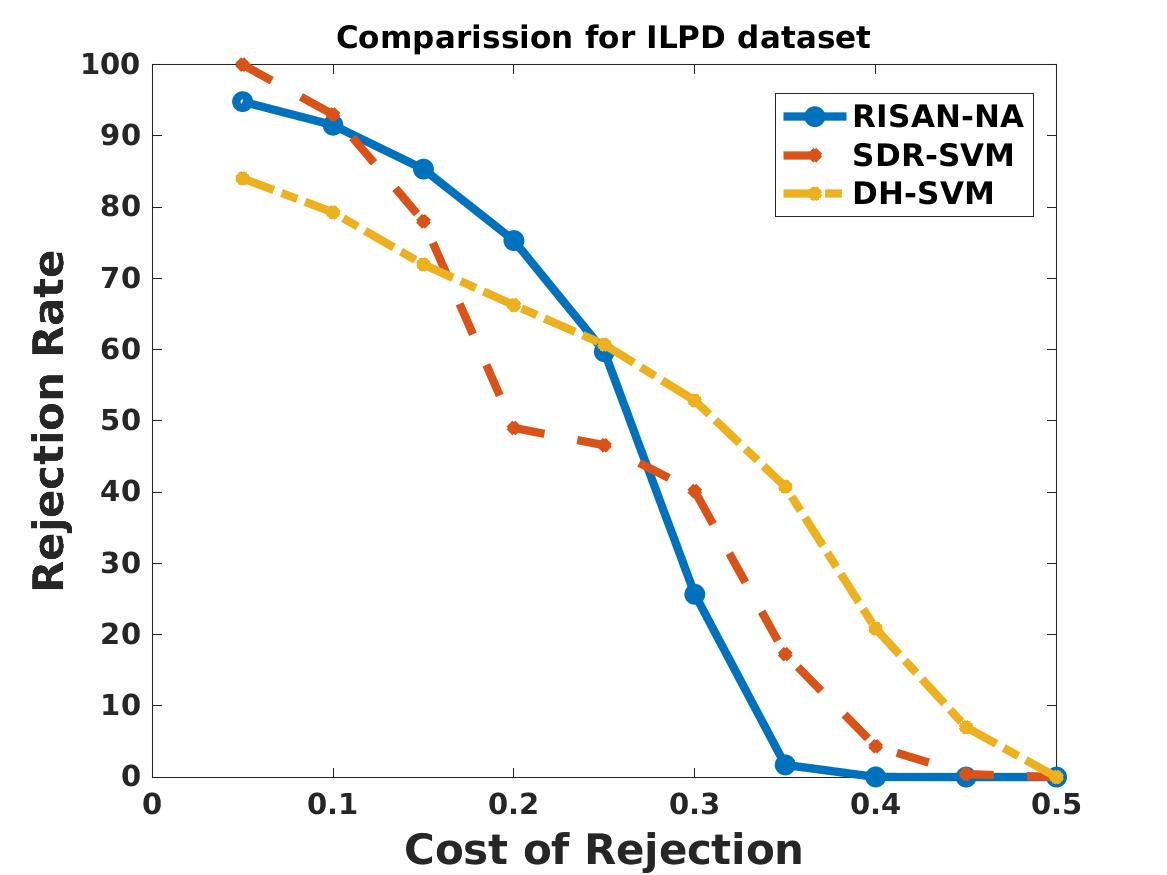}
  \label{fig:ilpd-rejection}}
  
  \hfill
 

  \caption{Small Dataset Results}
  \label{fig:Ilpd results}
\end{figure}


\begin{figure*}[h]
  \centering
  \subfloat[Cats vs Dogs Dataset ]{\includegraphics[width=0.2\textwidth]{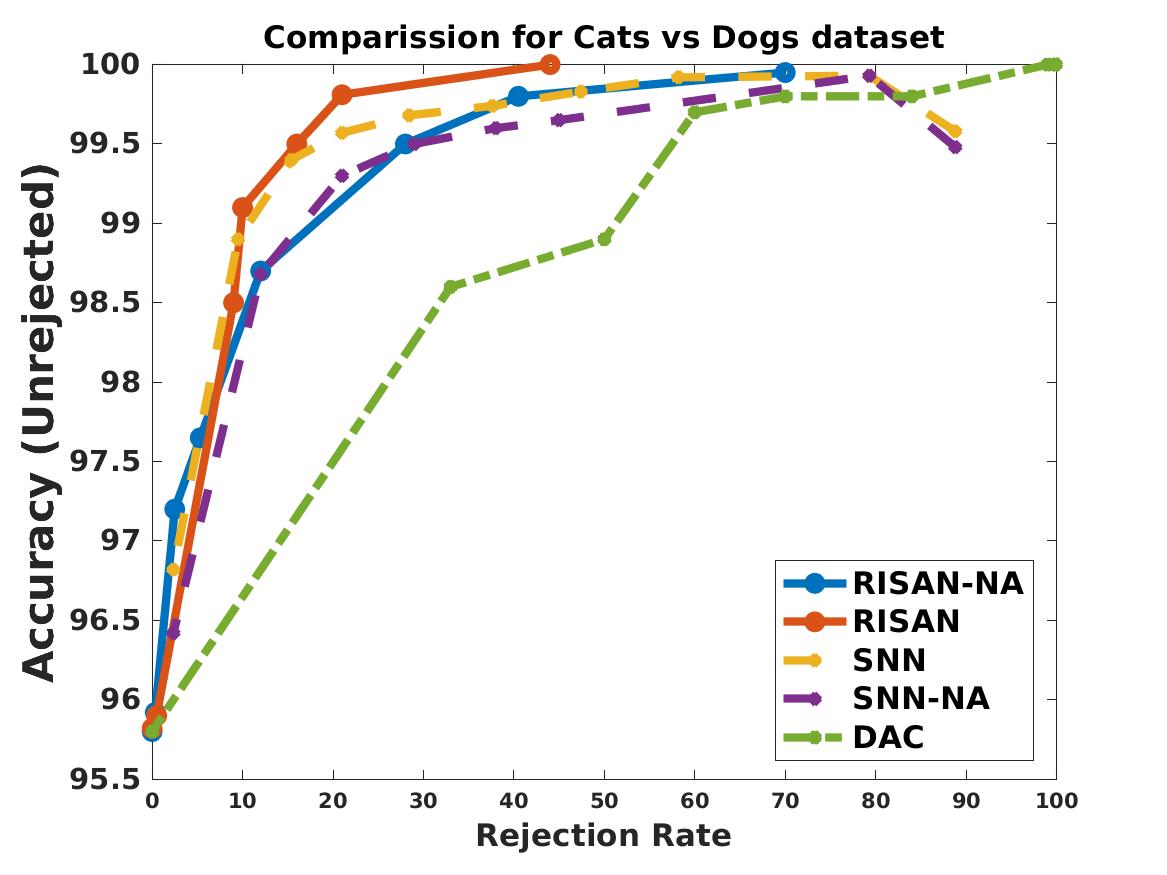}
  \label{fig:cat-dog}}
  \subfloat[CIFAR Dataset ]{\includegraphics[width=0.2\textwidth]{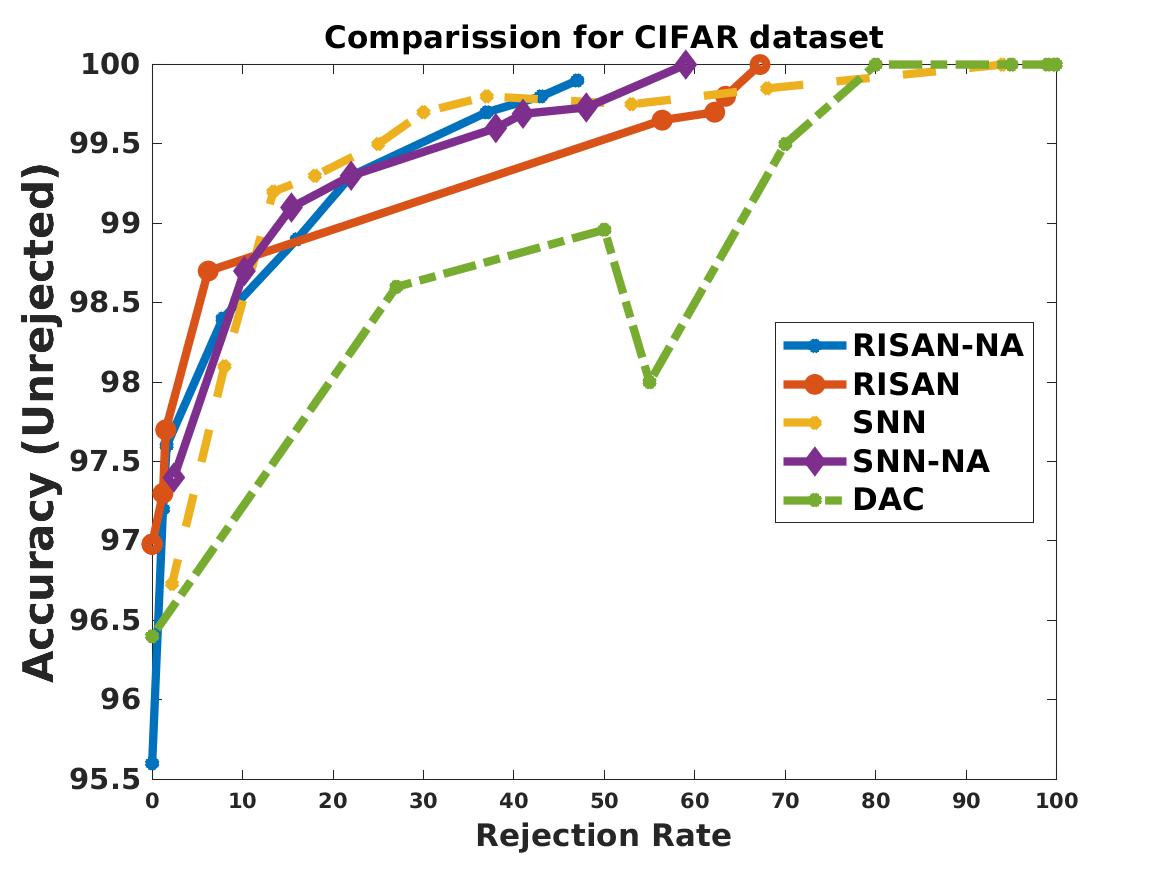}
  \label{fig:CIFAR-risk}}
  \subfloat[MNIST Dataset ]{\includegraphics[width=0.2\textwidth]{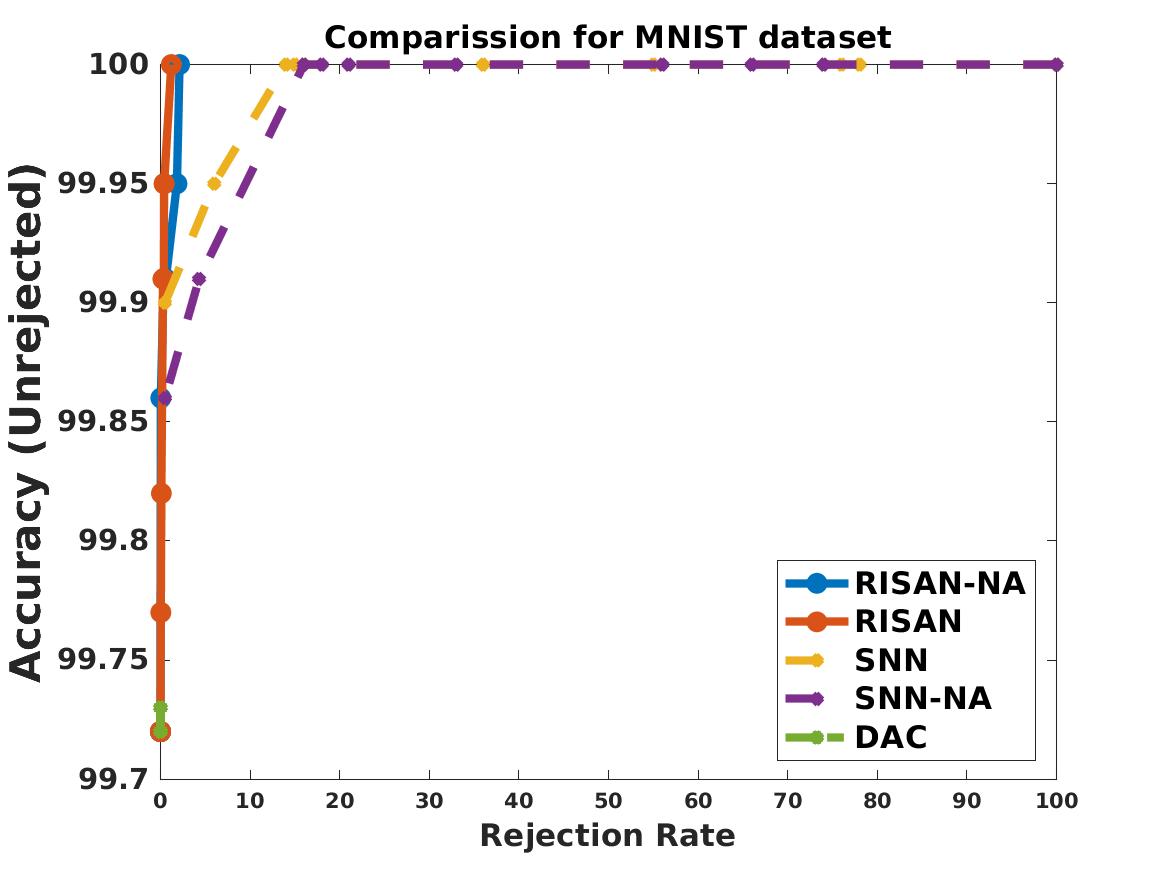}
  \label{fig:MNIST-risk}}
      \subfloat[CBIS-DDSM Dataset ]{\includegraphics[width=0.2\textwidth]{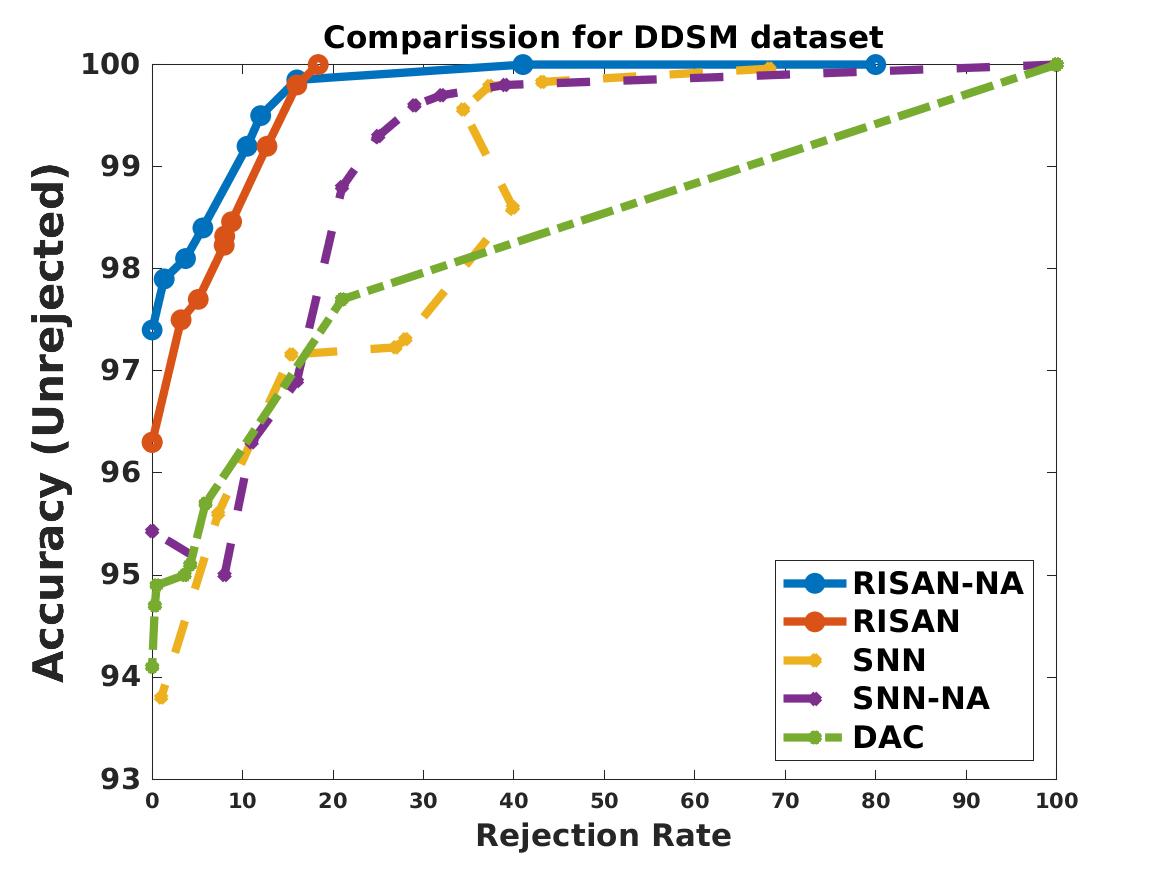}
  \label{fig:CBIS-risk}}
      \subfloat[Phishing Dataset ]{\includegraphics[width=0.2\textwidth]{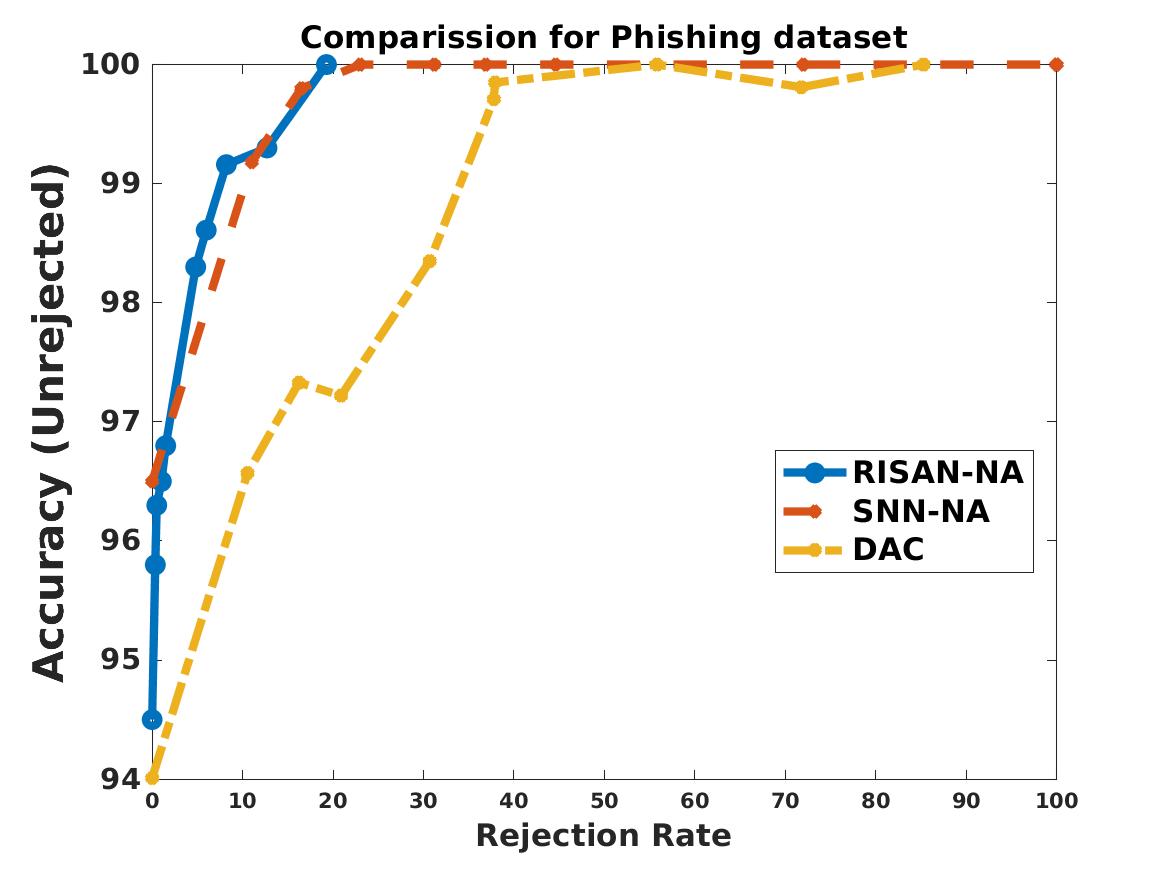}
  \label{fig:phish-risk}}
  \caption{Large Dataset Results}
  \label{fig:large-data}
\end{figure*}

\begin{figure*}[h]
  \centering
    \subfloat[Cats vs Dogs with 20\% \\ label noise ]{\includegraphics[width=0.22\textwidth]{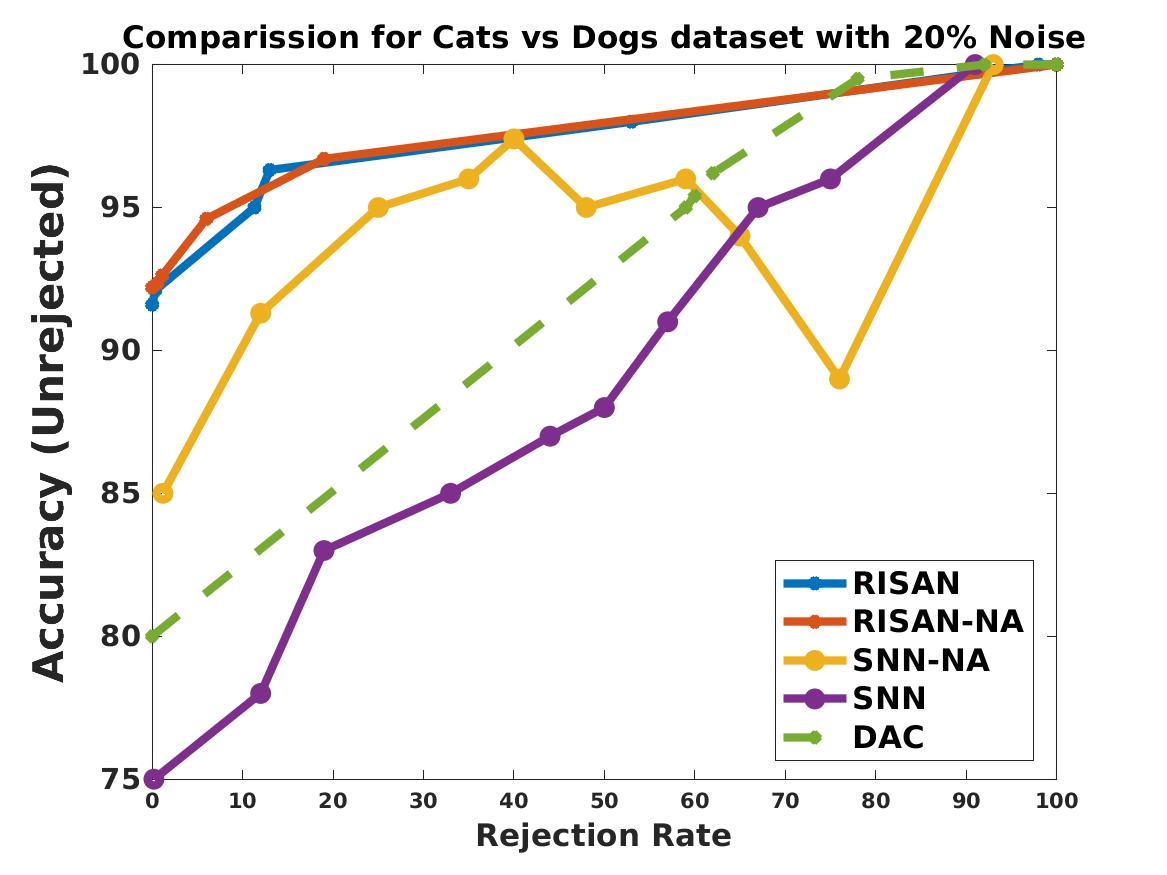}
  \label{fig:catdog-20}} \quad
  \subfloat[Cats vs Dogs with 40\% \\ label noise ]{\includegraphics[width=0.22\textwidth]{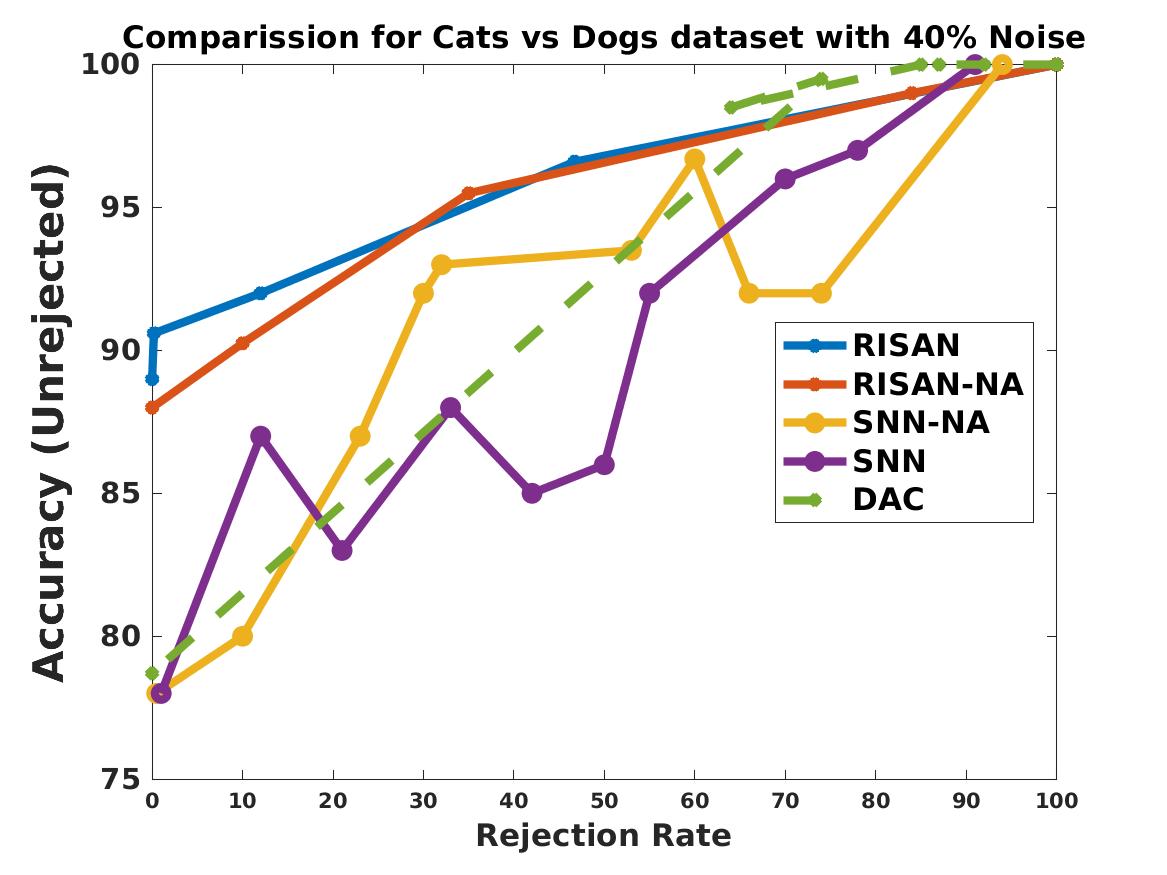}
  \label{fig:catdog-40}} \quad
  \subfloat[CIFAR with 20\% \\ label noise ]{\includegraphics[width=0.22\textwidth]{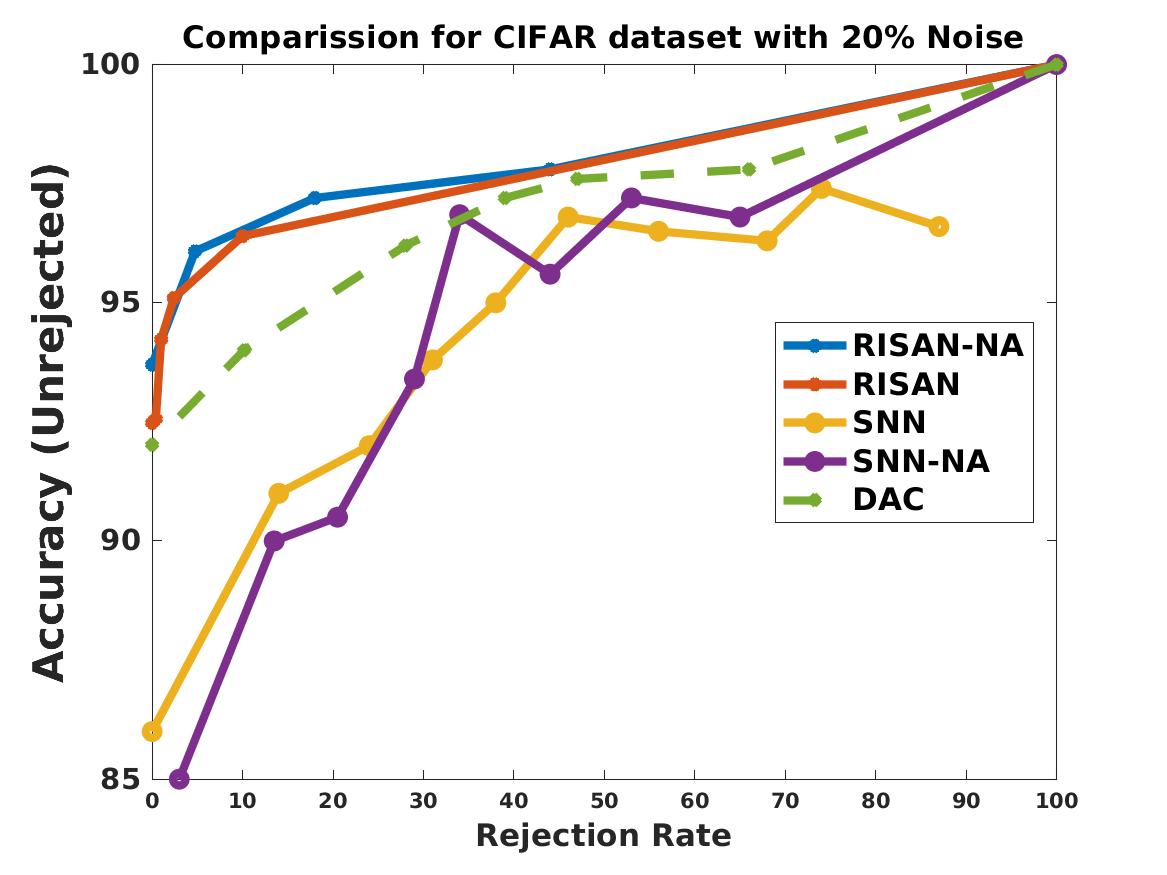}
  \label{fig:CIFAR-20}}
\subfloat[CIFAR with 40\% \\ label noise ]{\includegraphics[width=0.22\textwidth]{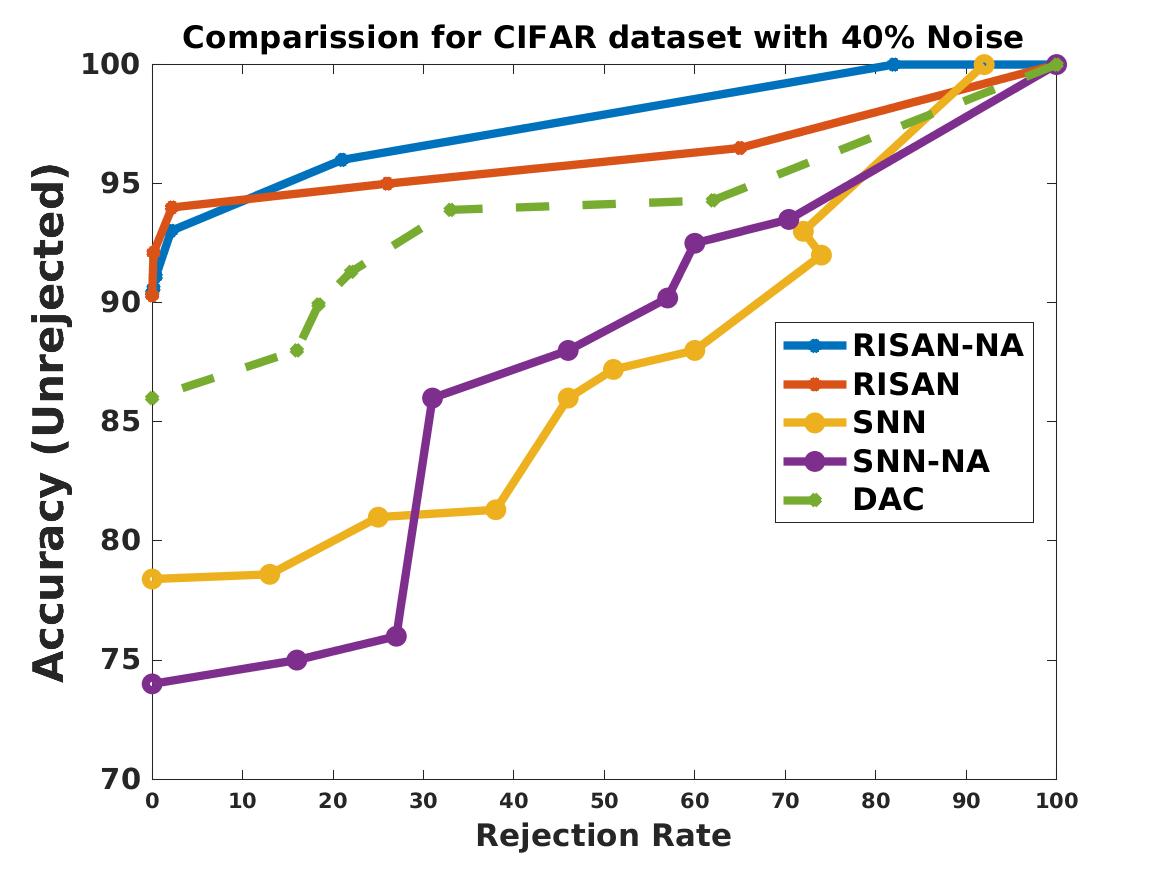}
  \label{fig:CIFAR-40}}
  
\caption{Comparison Results With Label Noise}
\label{fig:Robustness Results}
\end{figure*}

\subsection{Experimental Settings}
We execute experiments on ILPD and Ionosphere datasets in a 10-fold cross-validation fashion for 10 repetitions. We do these for the cost of rejection ($d$) varying from $[0.05,0.5]$ with a step size of 0.05. We monitor the accuracy (on unrejected samples), rejection rate, and the cross-validation risk $(0-d-1)$ for each value of $d$.
The experiments on large datasets compare five algorithms where each one takes a different parameter to introduce rejection. While DAC takes an abstention rate as input parameter, Selective Net (SNN) and SNN-NA take as input a coverage parameter. Here, coverage denotes fraction of points without abstention as the output label by the final trained classifier. For our networks RISAN and RISAN-NA, we have a cost of rejection which depends on the dataset. To get wide range of rejection rate, we choose cost of rejection $d$ parameter for our RISAN and RISAN-NA methods from set $\{ 0.0001,0.005,0.001,0.05,0.01,0.05,0.1,0.15,0.2,0.25,\\0.5 \}$. 
Both the abstention rate parameter for DAC and coverage parameter from SNN and SNN-NA are varied from [0.1,1.0] with a step size of 0.1. We plot the rejection rate vs accuracy plots to compare the five methods. The details of architectures and hyperparameters used in the experiments is given in Appendix \ref{sec:arc-details-hyp-selection}.
\subsection{Reproducibility}
	The code for the implementation would be available at \url{https://github.com/kalra20/RISAN-Robust-Instance-Specific-Abstain-Network} 
\subsection{Empirical Observations}
In Figure \ref{fig:Ilpd results}, we give results on smaller tabular datasets. We observed that proposed method achieves lower risk on the Ionosphere dataset (Figure~\ref{fig:iono-risk}) and performs comparably on the ILPD dataset except at a couple of points (Figure~\ref{fig:ilpd-risk}). Note that baseline methods on smaller datasets are optimized for small-sized datasets and fail to converge for large dataset. We perform better or comparable to such baseline methods. The proposed algorithm RISAN and RISAN-NA don't suffer from failing-to-converge issue on large datasets and perform comfortably to other neural network based algorithms (Figure~\ref{fig:large-data}).
We also make some interesting observations from results on the larger datasets. Both RISAN and RISAN-NA perform comparably on Cats vs. Dogs and CIFAR dataset with other datasets. However, RISAN performs slightly better than SNN, while RISAN-NA performs better than SNN-NA. This trend is consistent across all the datasets. We do acknowledge a consistent improvement of (1-2\%) in accuracy with the addition of an auxiliary loss.
We also observed that DAC fails to reject any examples for the MNIST dataset where the accuracy is too high ($99.7\%$) for VGG architecture despite complete coverage.
However, the RISAN and RISAN-NA perform better than SNN and SNN-NA while all four maintain a non-zero rejection rate. The fact that RISAN and RISAN-NA opting not to reject more samples even for an extremely small value of $d$ verifies that a cost-based abstain classifier is a more natural choice to learn the classifier than a coverage-based classifier. Since it chooses not to reject samples when the data is well separated, i.e., high accuracy without any rejection. 
This observation prompted the inspection of results on datasets with label noise. 
\section{Robustness of RISAN Against Label Noise}
In this section, we show the robustness results of RISAN against uniform label noise.
\paragraph{Experimental Setup:} We use Cats vs. Dogs and CIFAR 10 datasets for showing the robustness of RISAN against label noise. We introduce uniform label noise with a noise rate of $20\%$ and $40\%$. We ran the experiments with identical coverage values for SNN, SNN-NA, and DAC used in large dataset experiments. We used values of $d$ from set $\{0.05, 0.1, 0.15, \ldots, 0.4, 0.45, 0.5 \}$.
\paragraph{Results:} Results with label noise are shown in Figure~\ref{fig:Robustness Results}. We observe that with 20\% and 40\% label noise rates, RISAN and RISAN-NA performances do not drop much. On the other hand, the other approaches' performances drop significantly with label noise on both datasets. For 20\% label noise and low rejection rate, RISAN and RISAN-NA achieve at least $6-7$\% higher accuracy than other methods on both datasets. For 40\% label noise and low rejection rate, RISAN and RISAN-NA achieve around 10\% higher accuracy on the Cats vs. Dogs dataset and around 5\% higher accuracy on the CIFAR-10 dataset. As \cite{thulasidasan2018knows} claim that their approach (DAC) is robust to noisy labels, proposed algorithm RISAN improves around 5-10\% accuracy on unrejected samples from previously proposed robust learning algorithms. For large rejection rates, models are expected to get good accuracy on unrejected samples because the model is allowed to abstain large fraction of the data.

\begin{figure}[h]
  \centering
  \subfloat[Original  Image]{\includegraphics[width=0.2\textwidth]{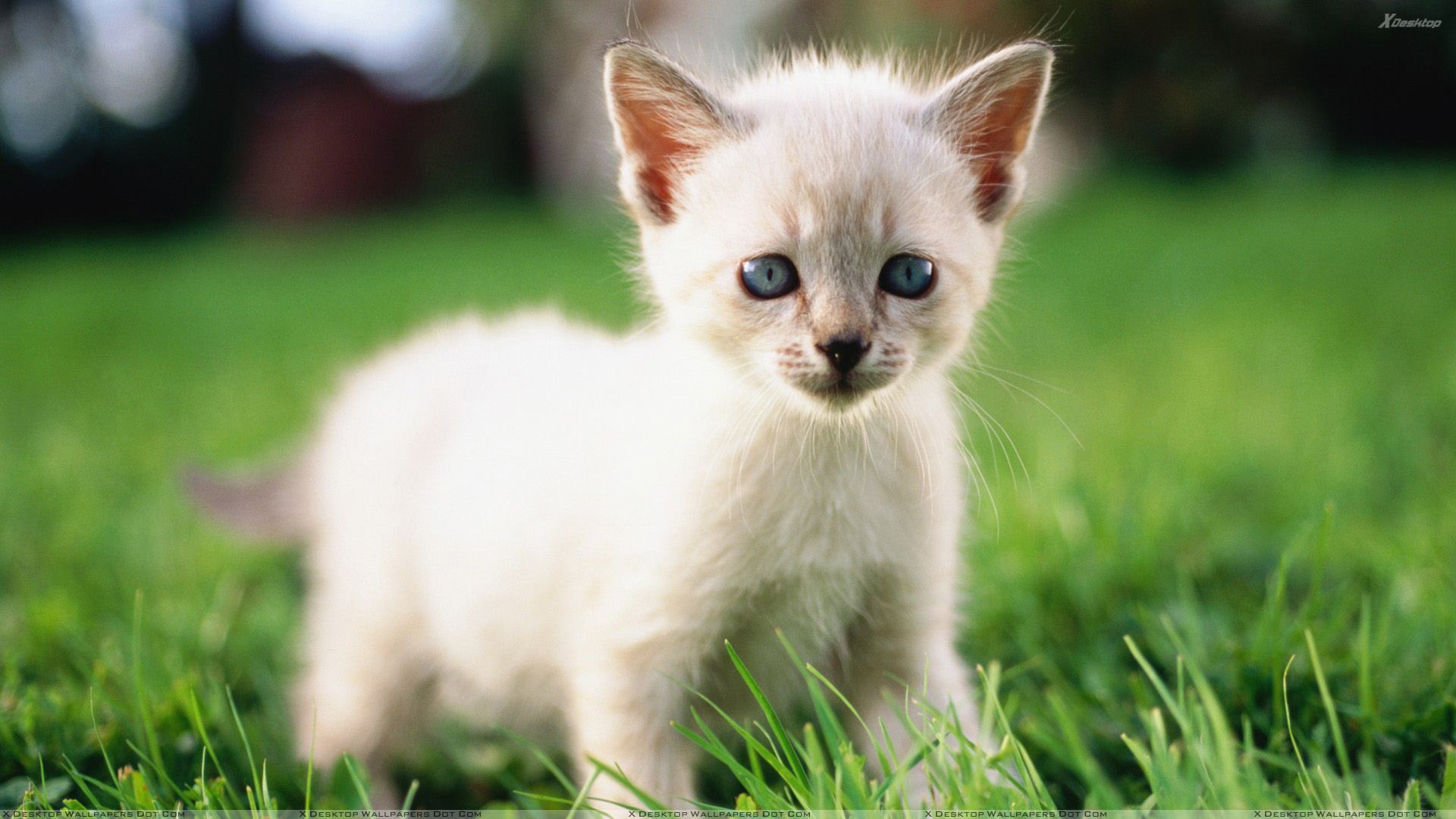}
  \label{fig:catasdog}} 
  \subfloat[Cat features \newline  highlighted]{\includegraphics[width=0.12\textwidth]{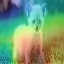}
  \label{fig:catasdog-cat}}  
  \subfloat[Dog features \newline  highlighted]{\includegraphics[width=0.12\textwidth]{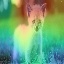}
  \label{fig:catasdog-dog}}
\\
  \centering
  \subfloat[Original Image]{\includegraphics[width=0.18\textwidth]{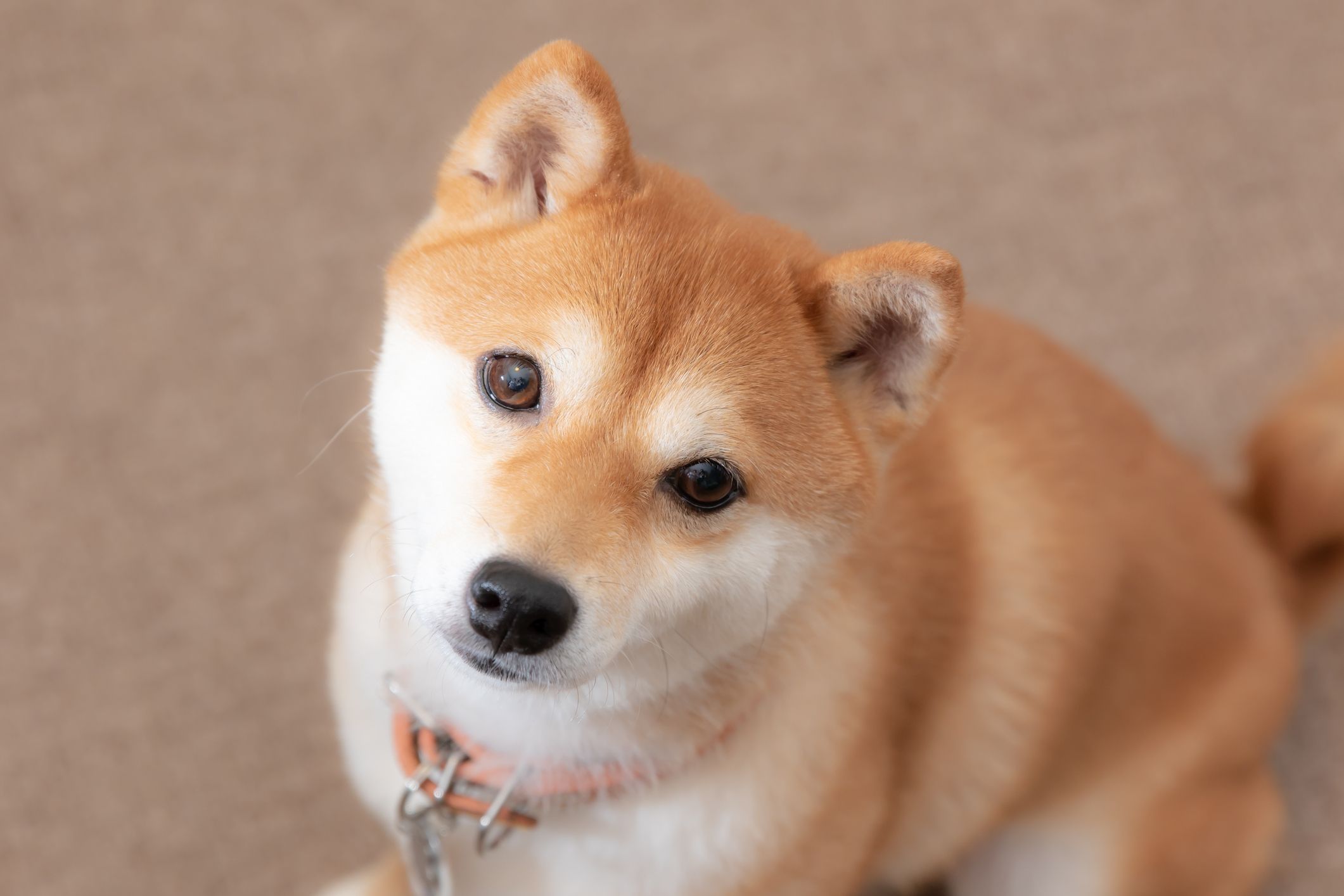}
  \label{fig:dogascat}}
  \subfloat[Cat features \newline highlighted]{\includegraphics[width=0.12\textwidth]{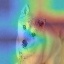}
  \label{fig:dogascat-cat}}
  \subfloat[Dog features \newline  highlighted]{\includegraphics[width=0.12 \textwidth]{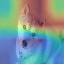}
  \label{fig:dogascat-dog}}
\\
  \centering
  \subfloat[Original Image]{\includegraphics[width=0.15\textwidth]{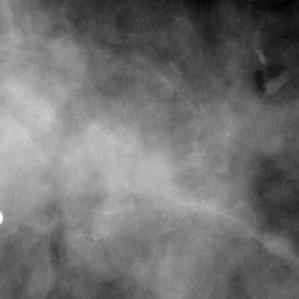}
  \label{fig:cancer-mass}}
  \subfloat[Negative \newline region highlighted]{\includegraphics[width=0.15\textwidth]{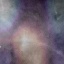}
  \label{fig:cancer-ben}}
  \subfloat[Positive \newline region highlighted]{\includegraphics[width=0.15 \textwidth]{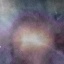}
  \label{fig:cancer-mal}}
  \caption{ GradCAM on image of a cat rejected by RISAN highlights what the network perceives as cat and dog regions in the image(a,b,c). GradCAM on image of a dog rejected by RISAN highlights what the network perceives as cat and dog regions in the image(d,e,f). GradCAM on image of a cancerous development rejected by RISAN highlights what the network perceives as negative and positive regions in the image(g,h,i)}
\end{figure}

\section{Explaining the Rejection Decisions}
In this section we introduced a visualization technique into the abstain network as a post processing step and examined the rejected examples.

\subsection{Representations learnt by RISAN}

In this section, we explored the following hypotheses  about the trained abstain network:
(i) Our network would reject images that contain pertinent features amongst both the classes
(ii) Prediction network will learn features that are more prominent and easily distinguishable for each class
(iii) The prediction network will give lesser precedence to features that are common to both classes. 
The implementation details of GradCAM in RISAN have been shifted to the Appendix \ref{sec:GradCAM}.
The GradCAM \cite{selvaraju2017grad} technique was used on the sigmoid outputs of the auxiliary head associated with the prediction network. Thus visualizing features learned by the prediction network to produce highlighted regions corresponding to the image's different classes.
We executed GradCAM on some selected examples that were ambiguous and tough to classify. Our network, as expected, choose to reject these samples. The images used in this task were re-scaled to 64x64 for the network to process the image.
In Fig. \ref{fig:catasdog}, we examined a cat image that could be mistaken for a dog. We observed that the cat's body, especially the legs, were majorly highlighted with reference to the \textit{cat} class in Fig. \ref{fig:catasdog-cat}. It's contrasted by the head region of the subject being highlighted in Fig. \ref{fig:catasdog-dog} with respect to \textit{dog} class. The legs and body region are important features for \textit{cat} class, as will be established in our later conducted experiments. In comparison, the head region of a \textit{dog} is equally important.
We examined another example, a dog in Fig. \ref{fig:dogascat} that can be mistaken for a cat.
We observed that subject's ear and the body is being majorly highlighted with reference to the \textit{cat} class in Fig. \ref{fig:dogascat-cat}. It's contrasted by the head region of the subject being highlighted in Fig. \ref{fig:dogascat-dog} with respect to \textit{dog} class.

In another example, we considered mammography of a malignant mass that's tough to spot and classify in Fig. \ref{fig:cancer-mass}. The network was trained to classify the presence of any irregularities(calcification or mass) in the image as \textit{positive}. We observed that in Fig. \ref{fig:cancer-mal}, the mass (irregular lighter region running through the image diagonally) is being highlighted with respect to the \textit{positive} class. It's contrasted by the larger region highlighted in Fig. \ref{fig:cancer-ben}, containing more surrounding \textit{negative} region, with respect to \textit{negative} class. Though there appears to be an overlap of highlighted regions, \textit{negative} class region focuses more on the surroundings of the mass while \textit{positive} class focuses more on the mass itself. But since features from both classes are present in the image, it's a good candidate for rejection.
Hence, we observed our prediction network highlighted pertinent features corresponding to each class found in the images and chose to reject these examples.

To verify our second and third assumptions, we then chose an interesting example of an animal, where a dog's body and head with a cat's legs and tail were infused. To compare and analyze our network's learned features, we also trained a separate network (CCEN) with categorical cross-entropy loss. We executed GradCAM on both the networks to compare the resulting highlighted regions in the images. We observed that while our network rejected the image, CCEN predicted the \textit{dog} class. When we examined the highlighted regions corresponding to different classes, we saw that in Fig. \ref{fig:catdog-1-cat} and Fig. \ref{fig:cat-dog-1-aux-cat} both networks chose to highlight the cat's legs fairly well in reference to \textit{cat} class. This is also coherent with our previous analysis of features highlighted for the \textit{cat} class.
However, when we analyzed Fig. \ref{fig:catdog-1-dog}, our network gave attention to the dog's body and head and less attention to the animal's legs. Whereas, as seen in Fig. \ref{fig:cat-dog-1-aux-dog}, CCEN pays attention to the animal's body and legs for \textit{dog} class. This holds with our belief that when rejecting, features corresponding to non-similar regions would get more attention and help make decisions only when the classifier is extremely certain.

\begin{figure}[h]
  \centering
  \subfloat[Original Image]{\includegraphics[width=0.18\textwidth]{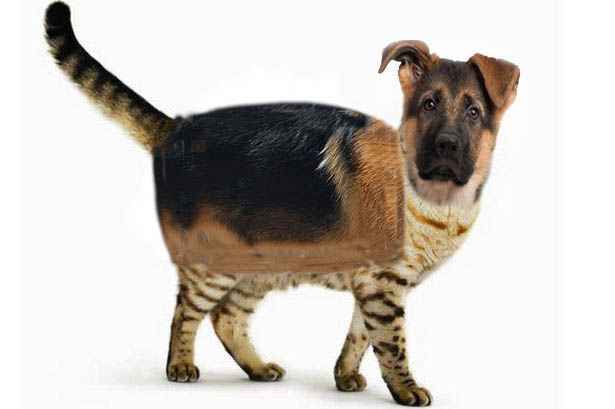}
  \label{fig:catdog-1}}
  \subfloat[Cat features \newline highlighted]{\includegraphics[width=0.12\textwidth]{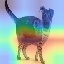}
  \label{fig:catdog-1-cat}} 
\quad
  \subfloat[Dog features \newline highlighted]{\includegraphics[width=0.12\textwidth]{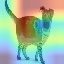}
  \label{fig:catdog-1-dog}}
\quad
  \subfloat[Original Image]{\includegraphics[width=0.18\textwidth]{Rejected Images/dogh-catb5.jpg}
  \label{fig:cat-dog-1-aux}}
  \subfloat[Cat features \newline highlighted]{\includegraphics[width=0.12\textwidth]{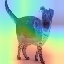}
  \label{fig:cat-dog-1-aux-cat}}
\quad
  \subfloat[Dog features \newline highlighted ]{\includegraphics[width=0.12\textwidth]{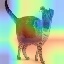}
  \label{fig:cat-dog-1-aux-dog}}
  \caption{The GradCAM on image of a cat's lower body and tail infused with a dog's upper body and head showcase the contrasting features learnt by RISAN (a,b,c) and CCEN (d,e,f)}
\end{figure}
\section{Conclusion and Future Work}
We introduced a novel implementation of double sigmoid loss in a deep neural network setting, RISAN for binary classification. We established the statistical properties of double sigmoid loss function such as classification calibration and excess risk bound. We also derived the generalization error bounds for input independent RISAN. We then demonstrated the various architectures and how each can be utilized for varied sized datasets. We also show that RISAN performs competitively to other state of the art shallow and deep neural network methods, SelectiveNet and Deep Abstaining Classifier in absence of noise but gains significant advantage when the data becomes noisy. We were also able to visualize the highlighted regions for corresponding classes in images and make inferences about rejected images. The results motivates the use of RISAN in applications where cost of misclassification is extremely high.

We leave a number of issues for future research such as extending the proposed method from binary classification to multiclass classification. Also, the study of representations learnt by other abstain neural networks and how they compare to RISAN is also an open future direction.





\bibliography{kalra_590}

\clearpage

\appendix
\section{Appendix}
\setcounter{theorem}{0}
\subsection{Proof of Theorem ~\ref{THM:CLASSCALIB}}
\label{Appendix:A1}
\begin{theorem}
For a fixed cost of rejection d, the risk under double sigmoid loss is minimized by the generalized Bayes classifier  $f_{d}^{∗}(.)$
\end{theorem}

\begin{proof}
The generalized bayes discriminant for reject option classifier (0-d-1 loss) is defined as \begin{equation}
    f^{*}_{d}(x) = 1\cdot\mathbb{I}_{\eta(x)>1-d}+0\cdot\mathbb{I}_{d\leq\eta(x)\leq1-d}-1\cdot\mathbb{I}_{\eta(x)<d}
\end{equation}
and the risk for double sigmoid loss is defined as, 
\begin{equation}
    R_{ds}(f,\rho) = 
    \mathbb{E}\left[L_{ds}(y f(\mathbf{x}), \rho)\right] \nonumber
\end{equation}
If $r_{\eta}(z)=\mathbb{E}_{y \mid \mathbf{x}}\left[L_{d r}(y f(\mathbf{x}), \rho)\right]$ and $z=f(\mathbf{x})$. Then
\begin{equation}
    r_{\eta}(z) = \eta L_{ds}(z, \rho) + (1-\eta)L_{ds}(-z,\rho) \nonumber
\end{equation}
or
\begin{equation}
    r_{\eta}(z) = 2(1-\eta) + (2\eta - 2d)\sigma(z+\rho) + 2(\eta + d - 1)\sigma(z-\rho) \nonumber
\end{equation}
This can also be written as,
\begin{align}
\label{eq:r-eta}
    r_{\eta}(z) = 2(1-\eta) + (\eta - d)\bigg(1-tanh\bigg(\frac{z+\rho}{2}\bigg)\bigg) \nonumber\\
    + (\eta + d - 1)\bigg(1-tanh\bigg(\frac{z-\rho}{2}\bigg)\bigg) 
\end{align}
where $\rho$ is the rejection parameter, $d$ is the cost of rejection and $\eta= P(Y=1|X)$. 
We observe that the function $r_{\eta}(z)$ can take different values corresponding to the value of parameter $\eta$. The parameter $\eta$ can be majorly broken into 3 intervals for the reject option classifier, $\eta\in[0,d]$, $\eta\in[d,1-d]$ and $\eta\in[1-d,1]$. 
Thus we study, $r_{\eta}(z)$ in these 3 intervals.
%
To find the minima of equation \ref{eq:r-eta}, we take it's derivative w.r.t. $z$.
First we expand \ref{eq:r-eta} using,
\begin{equation*}
    tanh(A\pm B) = \frac{tanh(A) \pm tanh(B)}{1 \pm tanh(A)tanh(B)}
\end{equation*}
with $A=\frac{\text{z}}{2}$ and $B=\frac{\rho}{2}$.
Further on differentiating w.r.t z we get,
\begin{equation}
    (K^{2}-1)(1-\zeta^{2})\frac{(2\eta-1)K
    ^{2}\zeta^{2}+(4d-2)K\zeta+2\eta-1}{\left(1-K^{2}\zeta^{2}\right)^{2}} 
    \label{eq:derivative}
\end{equation}
where $K=tanh(\frac{z}{2})$, $\zeta=tanh(\frac{\rho}{2})$.
We equate the derivative in equation \ref{eq:derivative} to $0$ and find solutions for K as $\pm1$ from $K^2-1=0$ and $\frac{(1-2d)\pm\sqrt{(1-2d)^{2}-(2\eta-1)^{2}}}{(2\eta-1)\zeta}$ from
\begin{equation*}
    (2\eta-1)K^{2}\zeta^{2}+(4d-2)K\zeta+2\eta-1
\end{equation*}

The numerator of eqn \ref{eq:derivative} contains two quadratic equations. We check if minima exists at $K=\pm1$ by taking the second derivative of eqn \ref{eq:derivative} and evaluating the sign at $K=\pm1$. We observe that the second derivative is positive hence minima exists for both the values.

However, for the roots $K = \frac{(1-2d)\pm\sqrt{(1-2d)^{2}-(2\eta-1)^{2}}}{(2\eta-1)\zeta}$ we look at the curve of the quadratic which yields the two solutions. We observe that the curve is opening upwards curve when $2\eta-1<0$ and opening downwards when $2\eta-1>0$ because $K^{2}-1\leq0$. 

These curves suggest $K_{1}=\frac{(1-2d)+\sqrt{(1-2d)^{2}-(2\eta-1)^{2}}}{(2\eta-1)\zeta}$ is the minima when $2\eta-1<0$ since the slope for $r_{\eta}(z)$ changes from negative to positive at $K_{1}$. Similarly, $K_{2}=\frac{(1-2d)-\sqrt{(1-2d)^{2}-(2\eta-1)^{2}}}{(2\eta-1)\zeta}$ is the minima when $2\eta-1>0$.


Thus $r_{\eta}^{*}(z)$ would be, 
\begin{align}
    r_{\eta}^{*}(z) = \min\begin{cases}
    2\eta \\
    1-(\eta-d)\left(\frac{T_{1}+\zeta^{2}(2\eta-1)}{(2\eta-1)\zeta+T_{1}\zeta}\right) \\ \quad -(\eta+d-1)\left(\frac{T_{1}-\zeta^{2}(2\eta-1)}{(2\eta-1)\zeta-T_{1}\zeta}\right)
    \\
    1-(\eta-d)\left(\frac{T_{2}+\zeta^{2}(2\eta-1)}{(2\eta-1)\zeta+T_{2}\zeta}\right) \\ \quad -(\eta+d-1)\left(\frac{T_{2}-\zeta^{2}(2\eta-1)}{(2\eta-1)\zeta-T_{2}\zeta}\right) \label{eqn:complex} \\
    2(1-\eta) 
    \end{cases} \nonumber
\end{align}
where $T_{1}=(1-2d)+\sqrt{(1-2d)^{2}-(2\eta-1)^{2}}$ and $T_{2}=(1-2d)-\sqrt{(1-2d)^{2}-(2\eta-1)^{2}}$.

 Moreover the complex roots $K_{1}$ and $K_{2}$ are real only when 
\begin{equation}
    (1-2d)^{2}-(2\eta-1)^{2}\geq0 \nonumber
\end{equation}
Therefore, the solutions $K_{1}$ and $K_{2}$ are real only when $d\leq\eta\leq1-d$.

Thus, when $\eta<d$, we have two candidates for minimum value. And we realise $2\eta$ is the minimum value since $2\eta \leq 2(1-\eta)$ and $\eta < d \leq 0.5$.
Similarly when $\eta>1-d$, we have two candidates for minimum value. And we observe, $2(1-\eta)$ would be the minimum value since $2(1-\eta)\leq2\eta$.

However, when $\eta\in[d,0.5]$ we have 3 candidates for minima
$2\eta$, $2(1-\eta)$ and \\ $1-(\eta-d)\left(\frac{T_{1}+\zeta^{2}(2\eta-1)}{(2\eta-1)\zeta+T_{1}\zeta}\right)-(\eta+d-1)\left(\frac{T_{1}-\zeta^{2}(2\eta-1)}{(2\eta-1)\zeta-T_{1}\zeta}\right)$. 
\textbf{Note}: Even though $1-(\eta-d)\left(\frac{T_{2}+\zeta^{2}(2\eta-1)}{(2\eta-1)\zeta+T_{2}\zeta}\right)-(\eta+d-1)\left(\frac{T_{2}-\zeta^{2}(2\eta-1)}{(2\eta-1)\zeta-T_{2}\zeta}\right)$ is a minima, it's only a minima when $2\eta-1>0$.

So we first show that, 
\begin{align}
    2\eta \geq 1-&(\eta-d)\left(\frac{T_{1}+\zeta^{2}(2\eta-1)}{(2\eta-1)\zeta+T_{1}\zeta}\right) \nonumber\\   &-(\eta+d-1)\left(\frac{T_{1}-\zeta^{2}(2\eta-1)}{(2\eta-1)\zeta-T_{1}\zeta}\right) \nonumber
\end{align}
Which can be rewritten  and compared as, 
 
\begin{align*}
    2(1-\eta) +& (\eta-d)(2) + (\eta+d-1)(2) 
    \geq \\  &2(1-\eta) + (\eta-d)\left(1-\frac{T_{1}+\zeta^{2}(2\eta-1)}{(2\eta-1)\zeta+T_{1}\zeta}\right)\\ +&(\eta+d-1)\left(1-\frac{T_{1}-\zeta^{2}(2\eta-1)}{(2\eta-1)\zeta-T_{1}\zeta}\right)
\end{align*}
\begin{align}
    &0\geq (\eta-d)\left(-1-\frac{T_{1}+\zeta^{2}(2\eta-1)}{(2\eta-1)\zeta+T_{1}\zeta}\right)\nonumber \\&\hspace{1 cm}+ (\eta+d-1)\left(-1-\frac{T_{1}-\zeta^{2}(2\eta-1)}{(2\eta-1)\zeta-T_{1}\zeta}\right)\nonumber\\
    &=(\eta-d)\left(\frac{-T_{1}(1+\zeta)-(2\eta-1)\zeta(\zeta+1)}{(2\eta-1)\zeta+T_{1}\zeta}\right)+\nonumber \\&\hspace{1cm} (\eta+d-1)\left(\frac{T_{1}(\zeta-1)+(2\eta-1)\zeta(\zeta-1)}{(2\eta-1)\zeta-T_{1}\zeta}\right) \nonumber\\
    &=\left(\frac{T_{1}\left((1-2d)T_{1}-(2\eta-1)^{2}\right)}{(2\eta-1)^{2}\zeta-T_{1}^{2}\zeta}\right)\nonumber
    \\&\hspace{1 cm}+\left(\frac{(2\eta-1)^{2}(\zeta)^{2}(T_{1}+(2d-1))}{(2\eta-1)^{2}\zeta-T_{1}^{2}\zeta}\right)+\nonumber
    \\&\hspace{2cm}\left(\frac{(2\eta-1)\zeta((2\eta-1)^{2}-T_{1}^{2})}{(2\eta-1)^{2}\zeta-T_{1}^{2}\zeta}\right) 
    \label{eq:cc-1} 
\end{align}
We also find further relationship between $2\eta-1$, $T_{1}$ and $T_{2}$ based on the value of $\eta$. We observe that $(2\eta-1)^{2}\leq T_{1}^{2}$ when $d<\eta\leq0.5$ and $(2\eta-1)^{2}\geq T_{2}^{2}$ when $0.5<\eta\leq 1-d$.
Using these facts we can see that equation \ref{eq:cc-1} holds true, even at maximum value of $\zeta=1$, for $\eta<0.5$.

Since $\eta<0.5$, $2\eta<2(1-\eta)$ and hence
$1-(\eta-d)\left(\frac{T_{2}+\zeta^{2}(2\eta-1)}{(2\eta-1)\zeta+T_{2}\zeta}\right)  -(\eta+d-1)\left(\frac{T_{2}-\zeta^{2}(2\eta-1)}{(2\eta-1)\zeta-T_{2}\zeta}\right)$ 
 is the minimum value of $r_{\eta}(z)$ when $d\leq\eta<0.5$.

Due to the symmetry of $r_{\eta}(z)$ we can similarly show that $1-(\eta-d)\left(\frac{T_{2}+\zeta^{2}(2\eta-1)}{(2\eta-1)\zeta+T_{2}\zeta}\right)  -(\eta+d-1)\left(\frac{T_{2}-\zeta^{2}(2\eta-1)}{(2\eta-1)\zeta-T_{2}\zeta}\right)$ is the minimum when $0.5<\eta\leq1-d$ by comparing it with $2(1-\eta)$. Since $2\eta>2(1-\eta)$ for $0.5<\eta\leq1-d$, $1-(\eta-d)\left(\frac{T_{2}+\zeta^{2}(2\eta-1)}{(2\eta-1)\zeta+T_{2}\zeta}\right)  -(\eta+d-1)\left(\frac{T_{2}-\zeta^{2}(2\eta-1)}{(2\eta-1)\zeta-T_{2}\zeta}\right)$ is the minimum value of $r_{\eta}(z)$ when $0.5<\eta\leq1-d$.
Following from above established minimum values for each region of $\eta$, the $r_{\eta}^{*}(z)$ becomes, 
\begin{align*}
    r_{\eta}^{*}(z) = \begin{cases}
    2\eta & \eta<d\\
    1-(\eta-d)\left(\frac{T_{1}+\zeta^{2}(2\eta-1)}{(2\eta-1)\zeta+T_{1}\zeta}\right)- \\ \hspace{0.5 cm}(\eta+d-1)\left(\frac{T_{1}-\zeta^{2}(2\eta-1)}{(2\eta-1)\zeta-T_{1}\zeta}\right) & d\leq\eta<0.5\\
    1-(\eta-d)\left(\frac{T_{2}+\zeta^{2}(2\eta-1)}{(2\eta-1)\zeta+T_{2}\zeta}\right)- \\ \hspace{0.5 cm}(\eta+d-1)\left(\frac{T_{2}-\zeta^{2}(2\eta-1)}{(2\eta-1)\zeta-T_{2}\zeta}\right) & 0.5<\eta\leq1-d\\
    2(1-\eta) & \eta>1-d
    \end{cases} \nonumber
\end{align*}
The reject option in reject option classifiers is exercised when $z \in (-\rho,\rho)$ thus we need to show that $z^{*}$ for double sigmoid loss for region $d<\eta<1-d$ lies in $-\rho,\rho$.
Let $\theta = 2\eta-1$ and we show that $\frac{-\rho}{2} \leq K_{1}\leq 0$ for $2d-1\leq \theta<0$ and $0\leq K_{2}\leq \rho$ for $0 < \theta \leq 1-2d$. 
\begin{align*}
    -\rho&\leq 2\tanh^{-1}(K_{1}) \\
    -\zeta &\leq \frac{1-2d + \sqrt{(1-2d)^{2}-\theta^{2}}}{\theta\zeta} \\
    -\theta\zeta^{2} &\leq 1-2d + \sqrt{(1-2d)^{2}-\theta^{2}}\\
    0 &\geq \theta^{2} +\theta^{2}\zeta^{4}+2(1-2d)\zeta^{2}
\end{align*}
which says that for $K_{1}\geq-\rho$, $\theta\in\left[\frac{-2(1-2d)\zeta^{2}}{1+\zeta^{4}},0\right]$. Similarly, for $K_{2}\leq\rho$ we get $\theta\in\left[\frac{2(1-2d)\zeta^{2}}{1+\zeta^{4}},0\right]$. 
We also verify that our current solutions of $\theta$ lie between $[2d-1,0)$ and $(0,1-2d]$.
\begin{align*}
    \frac{-2(1-2d)\zeta^{2}}{1+\zeta^{4})}&\geq 2d-1 \\
    \zeta^{2}(2-\zeta^{2})&\leq 1
\end{align*}
which is true for all $\zeta$. Similarly we can show that $\theta$ with respect to $K_{2}$ also lies in $(0,1-2d]$.
Thus our $z^{*}$ would become,
\begin{equation}
    z^{*} = \begin{cases}
    -\infty & \eta<d\\
    [-\rho,0) & d \leq \eta < 0.5\\
    (0,\rho] & 0.5 < \eta \leq 1-d\\
    \infty & \eta > 1-d
    \end{cases} \nonumber
\end{equation}
Thus, our $f_{ds}^{*}$ or the discriminant function for double sigmoid loss would be
\begin{equation}
    f_{ds}^{*} = \begin{cases}
    -1 & \eta<d\\
    0 & d \leq \eta \leq 1-d\\
    1 & \eta > 1-d
    \end{cases} \nonumber
\end{equation}
which is similar to the bayes discriminant function for 0-d-1 loss. Therefore, bayes discriminant function minimizes the double sigmoid risk.
\end{proof}

\subsection{Proof of Theorem \ref{THM:EXCESS}}
\label{Appendix:A2}
\begin{theorem}
Let $0\leq d \leq 1/2$ and a measurable function $z$. Then we have the excess risk relation as 
\begin{equation}
     \psi\left(R_{d}(f, \rho)-R_{d} (f_d^*)\right) \leq \left(R_{ds}(f, \rho) - R_{ds} ( f_d^* )\right) \nonumber
\end{equation}
where  
\begin{equation}
    \psi(\theta) = \begin{cases}
    0 & \theta = 0 \\
    (2d-1)\zeta+\left(\frac{\theta+1-2d}{2}\right)\left(\frac{T+\zeta^{2}\theta}{\zeta\theta+T\zeta}\right)\\\hspace{1.25 cm}+\left(\frac{\theta+2d-1}{2}\right)\left(\frac{T-\zeta^{2}\theta}{\zeta\theta-T\zeta}\right) & \theta \in (0,1-2d] \\
    \theta + (2d-1)\zeta & \theta \in [1-2d,1]
    \end{cases} \nonumber
\end{equation}
and $\theta = R_{d}(f, \rho) - R_{d}( f_d^* )$. Also, $\zeta=tanh(\frac{\rho}{2})$ and $T=(1-2d)-\sqrt{(1-2d)^{2}-\theta^{2}}$. 
\end{theorem}
\begin{proof}
We follow the approach described  in  \cite{bartlett2006convexity} and define the $\psi:[0, 1] \xrightarrow{} [0,\infty)$ transform of a loss function as $\psi(\theta) = $co$ \Tilde{\psi}(\theta)$, where
\begin{equation}
    {\Tilde{\psi}}(\theta) = H^{-}\bigg(\frac{1+\theta}{2}\bigg) - H\bigg(\frac{1+\theta}{2}\bigg) \nonumber
\end{equation}
and $co$ represents convex hull of the function.
This implies that $\psi =  \Tilde{\psi}$
if and only if $\Tilde{\psi}$ is convex.
Also, $\theta\in[0,1]$ with
\begin{equation}
    H^{-}(\eta) = \inf_{z(2\eta-1)\leq0} r_{\eta}(z) \quad \text{and} \quad H(\eta) = \inf_{z\in R} r_{\eta}(z)  \nonumber
\end{equation}

From the definition $H^{-}$ is the optimal conditional risk such that sign of z disagrees with sign of $2\eta-1$.  
\begin{align*}
     &H^{-}\bigg(\frac{1+\theta}{2}\bigg) = \inf_{z\theta\leq0} r_{\eta}(z)
      \\&= \inf_{z\in(-\infty,0)} 2(1-\eta) + (2\eta - 2d)\sigma(z+\rho) \\
      &\hspace{1 cm}+ 2(\eta + d - 1)\sigma(z-\rho)  \\
      &= \inf_{z\in(-\infty,0)} 1-\theta + (1+\theta - 2d)\sigma(z+\rho) \\&\hspace{1cm}+ (\theta+2d-1)\sigma(z-\rho)\\
      &= \inf_{z\in(-\infty,0)} 1-\theta + \bigg(\frac{1+\theta - 2d}{2}\bigg)\bigg(1-tanh\bigg(\frac{z+\rho}{2}\bigg)\bigg) \nonumber \\&\hspace{2cm} + \bigg(\frac{\theta+2d-1}{2}\bigg)\bigg(1-tanh\bigg(\frac{z-\rho}{2}\bigg)\bigg) \nonumber\\
      &= 1-\theta + \bigg(\frac{1+\theta - 2d}{2}\bigg)\bigg(1-tanh\bigg(\frac{\rho}{2}\bigg)\bigg) \\&\hspace{2 cm}+ \bigg(\frac{\theta+2d-1}{2}\bigg)\bigg(1-tanh\bigg(\frac{-\rho}{2}\bigg)\bigg) \nonumber 
\end{align*}
Let $\tanh(\frac{\rho}{2}) = \zeta$ and thus $tanh(\frac{-\rho}{2})=-\zeta$ 
\begin{align*}
     H^{-}\bigg(\frac{1+\theta}{2}\bigg) &= 1-\theta + \bigg(\frac{1+\theta - 2d}{2}\bigg)(1-\zeta) \\&\hspace{2 cm}+ \bigg(\frac{\theta+2d-1}{2}\bigg)(1+\zeta)  \\
      &= 1 -\zeta +2d\zeta \nonumber
\end{align*}
Similarly, from the definition $H$ is the optimal conditional risk,
\begin{align*}
    &H\bigg(\frac{1+\theta}{2}\bigg) = \inf_{z\in R} r_{\eta}(z)  \nonumber\\
     &= \inf_{z\in R} 2(1-\eta) + (2\eta - 2d)\sigma(z+\rho) + 2(\eta + d - 1)\sigma(z-\rho)  \nonumber \\
     &= \inf_{z\in R} 1-\theta + \bigg(\frac{1+\theta - 2d}{2}\bigg)\bigg(1-tanh\bigg(\frac{z+\rho}{2}\bigg)\bigg) \nonumber \\&\quad \quad + \bigg(\frac{\theta+2d-1}{2}\bigg)\bigg(1-tanh\bigg(\frac{z-\rho}{2}\bigg)\bigg) \nonumber
\end{align*}
\begin{align*}
      &\hspace{0cm}H\bigg(\frac{1+\theta}{2}\bigg) = \inf_{z \in R} 1-\theta+ \\&\hspace{0 cm}\left(\frac{1+\theta-2d}{2}\right) \left(\frac{1+tanh(\frac{z}{2})tanh(\frac{\rho}{2})-tanh(\frac{z}{2})-tanh(\frac{\rho}{2})}{1
     +tanh(\frac{z}{2})tanh(\frac{\rho}{2})}\right)\\&\hspace{0 cm}+ \bigg(\frac{\theta+2d-1}{2}\bigg)\left(\frac{1-tanh(\frac{z}{2})tanh(\frac{\rho}{2})-tanh(\frac{z}{2})+tanh(\frac{\rho}{2})}{1-tanh(\frac{z}{2})tanh(\frac{\rho}{2})}\right)
\end{align*}
Since $r_{\eta}^{*}(z) = H(\eta)$, we follow the definition of $r_{\eta}^{*}(z)$, hence 
\begin{align*}
    \hspace{0cm}H(\eta)=r_{\eta}^{*}(z) = \begin{cases}
    2\eta & \eta<d\\
    1-(\eta-d)\left(\frac{T+\zeta^{2}(2\eta-1)}{(2\eta-1)\zeta+T\zeta}\right) \\ -(\eta+d-1)\left(\frac{T-\zeta^{2}(2\eta-1)}{(2\eta-1)\zeta-T\zeta}\right) & d\leq\eta\leq1-d\\
    2(1-\eta) & \eta>1-d
    \end{cases} \nonumber
\end{align*}
Also, since $\theta\geq0$ and $\eta = \frac{1+\theta}{2}$, we use definition of $r_{\eta}^{*}(z)$ for $\eta\geq0.5$. 
Thus, $H\bigg(\frac{1+\theta}{2}\bigg)$ is defined over different intervals as,
\begin{align*}
    H\bigg(\frac{1+\theta}{2}\bigg) = \begin{cases}
    1+(2d-1)\zeta & \theta = 0 \\
    1-\left(\frac{\theta+1-2d}{2}\right)\left(\frac{T+\zeta^{2}\theta}{\zeta\theta+T\zeta}\right)\\\hspace{0.2cm}-\left(\frac{\theta+2d-1}{2}\right)\left(\frac{T-\zeta^{2}\theta}{\zeta\theta-T\zeta}\right) & \theta \in (0,1-2d] \\
    1-\theta & \theta \in [1-2d,1]
    \end{cases} \nonumber
\end{align*}
where $T = (1-2d)-\sqrt{(1-2d)^{2}-\theta^{2}}$ and $\zeta = tanh(\frac{\rho}{2})$.

Thus, our $\Tilde{\psi}(\theta)$ would be
\begin{equation}
    \Tilde{\psi}(\theta) = \begin{cases}
    0 & \theta = 0 \\
    (2d-1)\zeta+\left(\frac{\theta+1-2d}{2}\right)\left(\frac{T+\zeta^{2}\theta}{\zeta\theta+T\zeta}\right)\\ \hspace{0.2cm}+ \left(\frac{\theta+2d-1}{2}\right)\left(\frac{T-\zeta^{2}\theta}{\zeta\theta-T\zeta}\right) & \theta \in (0,1-2d] \\
    \theta + (2d-1)\zeta & \theta \in [1-2d,1]
    \end{cases} \nonumber
\end{equation}
The function $\Tilde{\psi}(\theta)$ is continuous in $\theta$. Moreover, the corresponding term of $\Tilde{\psi}(\theta)$ to $\theta \in(0,1-2d]$, achieves indeterminate values at $\theta=0$ and $\theta=1-2d$, which is resolved by finding the limit value, which shows the continuity of $\Tilde{\psi}(\theta)$.
\begin{align*}
    &\lim_{\theta\rightarrow 0} \left(\frac{\theta+1-2d}{2}\right)\left(\frac{T+\zeta^{2}\theta}{\zeta\theta+T\zeta}\right)\\&\hspace{4cm}+\left(\frac{\theta+2d-1}{2}\right)\left(\frac{T-\zeta^{2}\theta}{\zeta\theta-T\zeta}\right)\\
    &=\frac{1}{2}\left(\frac{\theta T'+2\zeta^{2}\theta}{\zeta+\zeta T'}\right)+\frac{1-2d}{2}\left(\frac{T' + \zeta^{2}}{\zeta+\zeta T'}\right) \\
    &\hspace{2 cm}+\frac{1}{2}\left(\frac{\theta T'-2\zeta^{2}\theta}{\zeta-\zeta T'}\right)+\frac{1-2d}{2}\left(\frac{T' - \zeta^{2}}{\zeta-\zeta T'}\right) \\ 
    &=\frac{1}{2}\left(\frac{\theta^{2}+2\zeta^{2}\theta\sqrt{(1-2d)^{2}-\theta^{2}}}{\zeta\sqrt{(1-2d)^{2}-\theta^{2}}+\zeta \theta}\right)\\&\hspace{1 cm}+\frac{1-2d}{2}\left(\frac{\theta + \zeta^{2}\sqrt{(1-2d)^{2}-\theta^{2}}}{\zeta\sqrt{(1-2d)^{2}-\theta^{2}}+\zeta\theta}\right) \\
    &+\frac{1}{2}\left(\frac{\theta^{2}-2\zeta^{2}\theta\sqrt{(1-2d)^{2}-\theta^{2}}}{\zeta\sqrt{(1-2d)^{2}-\theta^{2}}-\zeta \theta}\right)\\&\hspace{1 cm}+\frac{2d-1}{2}\left(\frac{\theta - \zeta^{2}\sqrt{(1-2d)^{2}-\theta^{2}}}{\zeta\sqrt{(1-2d)^{2}-\theta^{2}}-\zeta\theta}\right) \\ 
    \\&= \frac{1-2d}{2}(\zeta) + \frac{2d-1}{2}(-\zeta) = (1-2d)\zeta
\end{align*}
where $T'=\frac{dT}{d\theta}=\frac{\theta}{\sqrt{(1-2d)^{2}-\theta^{2}}}$. However, at $\theta=1-2d$ we can first need to look at value of $H(\frac{1+\theta}{2})$ at $\theta=1-2d$,
The value of $K_{1}$ at $\theta=1-2d$,
\begin{equation*}
    K_{1} = \frac{(1-2d)-\sqrt{(1-2d)^{2}-(1-2d)^{2}}}{(1-2d)\zeta} = \frac{1}{\zeta}
\end{equation*}
which means $K_{1}$ will be valid only when $\zeta=1$ since $K_{1}\leq1$.
Thus,
\begin{equation*}
    \lim_{\theta \rightarrow 1-2d} \zeta = 1
\end{equation*}
Using this information when finding the value at the limit
\begin{align*}
    \lim_{\theta\rightarrow 1-2d} &\left(\frac{\theta+1-2d}{2}\right)\left(\frac{T+\zeta^{2}\theta}{\zeta\theta+T\zeta}\right)\\&\hspace{2 cm}+\left(\frac{\theta+2d-1}{2}\right)\left(\frac{T-\zeta^{2}\theta}{\zeta\theta-T\zeta}\right) \\= &\frac{\theta+1-2d}{2}(1)+\frac{\theta+2d-1}{2}(-1) = 1-2d = \theta
\end{align*}
We can easily see that for the intervals $\theta=0$ and $\theta\in[1-2d,1]$, $\Tilde{\psi}(\theta)$ is convex. However, we show the convexity for the interval $\theta\in(0,1-2d)$ by taking the second derivative of the corresponding $\Tilde{\psi}(\theta)$. 

We first show the convexity of $C_{1}=\left(\frac{\theta+1-2d}{2}\right)\left(\frac{T+\zeta^{2}\theta}{\zeta\theta+T\zeta}\right)$.
Since our functions both the numerator $f(\theta)=(\theta+1-2d)(T+\zeta^{2}\theta)$ and denominator $g(\theta)=(2)(\zeta\theta+T\zeta)$ are convex and $g(\theta)>0$. We can say that $C_{1}$ is convex if 
\begin{align}
    \left(\frac{f(\theta)}{g(\theta)}\right)^{''} = \frac{f''g^{2}-2f'gg'-fgg''+2f(g')^{2}}{g^{3}}\geq 0
    \label{eqn:nd}
\end{align}
i.e. the numerator of eqn \ref{eqn:nd} is greater than 0. Let $M = \theta+1-2d$, then 
\begin{align*}
    &(2T'+MT''+2\zeta^{2})(2\zeta)^{2}(\theta+T)^{2}\\&-2(T+MT'+\zeta^{2}M+\zeta^{2}\theta)(2\zeta)^{2}(\theta+T)(1+T')\\&\hspace{0 cm}-(MT+M\zeta^{2}\theta)(2\zeta)^{2}(\theta+T)T''\\&+2(MT+\zeta^{2}M\theta)(2\zeta)^{2}(1+T')^{2} \geq 0 
    \end{align*}
On further solving we get,
    \begin{align*}
    &4\zeta^{2}(\theta+T)(MT''\theta)(1-\zeta^{2})+2(4\zeta^2)(\theta+T)(1-\zeta^2)(T'\theta-T)
    \\&\hspace{2 cm}+2(4\zeta^2)(1+T')(1-\zeta^2)M(T-T'\theta) \geq 0\\
    &4\zeta^{2}(\theta+T)(MT''\theta)(1-\zeta^{2})\\&\hspace{1 cm}+2(4\zeta^2)(1-\zeta^2)(M)(T-T'\theta)(1+T'-\theta-T)\geq0\\
    &(4\zeta^2)(1-\zeta^2)M[(\theta+T)(T''\theta)+2(T-T'\theta)(1+T')\\&\hspace{2 cm}+2(\theta+T)(T'\theta-T)]\geq 0
\end{align*}
This can be rearranged to get,
\begin{align*}
    &(4\zeta^2)(1-\zeta^2)M[(\theta+T)(2T(1-T)\\&\hspace{1 cm}+2T(T'-\theta)+2\theta(T'\theta-T')\\&\hspace{2 cm}+(T''T\theta+T''\theta^{2}+2TT'\theta-2(T')^{2}\theta))] \geq 0
\end{align*}

which is true for $\theta\in(0,1-2d]$.
Where 
\begin{equation*}
    T'=\frac{\theta\left((1-2d)-T\right)}{(1-2d)^{2}-\theta^{2}} 
    \quad \text{and} \quad
    T'' = \frac{(1-2d)^2\left((1-2d)-T\right)}{((1-2d)^{2}-\theta^{2})^{2}} 
\end{equation*}
with $T''\geq T'\geq T$ and $T'\geq \theta \geq T$.
These definitions can be used to verify  \begin{align}
    T''\theta^{2}+2T-2T^{2}-2T'\theta\geq 0 \label{eq:convex1}\\
    T''T\theta+2T'\theta^{2} -2\theta(T')^{2}\geq 0 \label{eq:convex2}\\
    2TT'+2TT'\theta -2T\theta \geq 0 \label{eq:convex3}
\end{align}
The inequality in eq. \ref{eq:convex3} is straightforward using the conditions on $T,T',T''$ and $\theta$. The quadratic inequality in eq. \ref{eq:convex1} is a upward opening curve and the solutions are at $\theta=0$. Hence, eq. \ref{eq:convex1} holds true. The same goes for eq. \ref{eq:convex2}, which is a quadratic in $\theta$, an upward opening curve with solutions at $\theta=0$.

Similarly, we can show the convexity of $C_{2}=\left(\frac{\theta-1+2d}{2}\right)\left(\frac{T-\zeta^{2}\theta}{\zeta\theta-T\zeta}\right)$.
Also, since sum of convex functions is a convex function, we establish that $\Tilde{\psi}(\theta)$ is convex when $\theta\in(0,1-2d)$.

Now, $\Tilde{\psi}(\theta)$ is individually convex in all the 3 intervals of $\theta$, $\psi(\theta) = \Tilde{\psi}(\theta)$ and continuous in $\theta$. The convexity of $\Tilde{\psi}(\theta)$ also depends on the slope of $\Tilde{\psi}(\theta)$ for these intervals. While $\Tilde{\psi}(\theta)$ has a slope of $0$ when $\theta=0$ and $1$ when $\theta\in[1-2d,1]$. The slope of $\Tilde{\psi}(\theta)$ for $\theta\in(0,1-2d]$ should be between $(0,1]$ since it's an increasing convex function which will achieve it's maximum at $\theta=1-2d$. 
So,
\begin{align*}
    &\lim_{\theta \to 1-2d} (2d-1)\zeta+\left(\frac{\theta+1-2d}{2}\right)\left(\frac{T+\zeta^{2}\theta}{\zeta\theta+T\zeta}\right)\\
    &\hspace{0.5 cm}+ \left(\frac{\theta+2d-1}{2}\right)\left(\frac{T-\zeta^{2}\theta}{\zeta\theta-T\zeta}\right)
    = (1-2d)+(2d-1)\zeta
\end{align*}
which is equal to $\theta+(2d-1)\zeta$ when $\theta=1-2d$, slope at $\theta=1-2d$ is 1 for $\Tilde{\psi}(\theta)$ corresponding to $\theta=1-2d$.
We can say that $\Tilde{\psi}(\theta)$ is convex in it's domain $\theta\in[0,1]$. Thus, $\Tilde{\psi}(\theta)=\psi(\theta)$ and this suggests our excess risk relationship is
\begin{equation}
     \psi\left(R_{d}(f, \rho)-R_{d} (f_d^*) \right) \leq \left(R_{ds}(f, \rho) - R_{ds}( f_d^* )\right) \nonumber
\end{equation}
where  
\begin{equation}
    \psi(\theta) = \begin{cases}
    0 & \theta = 0 \\
    (2d-1)\zeta+\left(\frac{\theta+1-2d}{2}\right)\left(\frac{T+\zeta^{2}\theta}{\zeta\theta+T\zeta}\right)\\\hspace{1.25 cm}+\left(\frac{\theta+2d-1}{2}\right)\left(\frac{T-\zeta^{2}\theta}{\zeta\theta-T\zeta}\right) & \theta \in (0,1-2d] \\
    \theta + (2d-1)\zeta & \theta \in [1-2d,1]
    \end{cases} \nonumber
\end{equation}
\end{proof}
\subsection{Proof of Theorem \ref{THM:GENBOUND}}
\label{Appendix:A3}

\begin{theorem}
Let $\mathcal{D}$ be any distribution on $\mathcal{X}\times \{-1,+1\}$. Let $0<\delta\leq 1$. Then for any $n$, $q\geq1$, $1\leq p <\infty$ and any set $S=\{\mathbf{x}_{1},\ldots,\mathbf{x}_{m}\}$; with probability at least $1-\delta$ (over $S\sim \mathcal{D}^m$), all functions $f\in \mathcal{F}$ satisfy

\begin{align}
R_{ds}(f, \rho) &\leq \Rhds{f}{\rho}
+\frac{\urho}{\sqrt{m}}+ \sqrt{\frac{8\ln \left(\frac{4}{\delta}\right)}{m}} + \sqrt{\frac{2\ln \left(\frac{2}{\delta}\right)}{m}} \nonumber \\&\hspace{-0.8 cm}+
\left(\frac{2\beta}{\sqrt{m}}\max_{i}\|\mathbf{x}_{i}\|_{\pc}\right)\left(2H^{\left[\frac{1}{\pc}-\frac{1}{q}\right]_{+}}\right)^{n-1} \nonumber
\end{align}
where $n$ is the number of layers in the network, $H$ is the number of neurons in the hidden layers, rejection region parameter is bounded as $\rho\leq\urho$. Also $\frac{1}{\pc}+\frac{1}{p}=1$ and $[a]_+=\max(0,a)$. $\Rhds{f}{\rho}$ is the empirical error and $\beta_{p,q}(W) = \prod_{k=1}^{n}\|W_{k}\|_{p,q}\leq \beta$.
\end{theorem}
\begin{proof}
We follow lemma \ref{lemma:mcdarmid},
\begin{align*} R_{ds}(f, \rho) \leq \Rhds{f}{\rho}&+2 L \hat{R}_{m}(\mathcal{F})\\&+2 B \sqrt{\frac{\ln \left(\frac{4}{\delta}\right)}{2 m}}+\left(b-a\right)\sqrt{\frac{ln\left(\frac{2}{\delta}\right)}{m}} \nonumber
\end{align*}
where $R(\mathcal{F})$ is the rademachar complexity and $f_{s}$ is a function belonging to function class $\mathcal{F}$. Since the bounds are described for loss $\ell: \mathcal{Y} \times[a, b] \rightarrow[0, B]$. For double sigmoid loss, we get $B = 2$,$a = -1$ and $b = 1$. So, we bound the generalization error with probability atleast $1-\delta$ by
\begin{equation}R_{ds}(f, \rho) \leq \Rhds{f}{\rho}+2 L \hat{R}_{m}(\mathcal{F})+ \sqrt{\frac{8\ln \left(\frac{4}{\delta}\right)}{m}}+\sqrt{\frac{2\ln \left(\frac{2}{\delta}\right)}{m}} \nonumber
\end{equation}
We now find an upper bound for the rademachar complexity $ \hat{R}_{m}(\mathcal{F})$, following theorem 1 in \cite{neyshabur2015norm}.

Hence, let $\hat{R}_{m}(\mathcal{F})= R(\mathcal{F}_{\beta_{p,q}\leq\beta}^{n,H})$ where $\mathcal{F}(\mathbf{x}) = \left|\mathbf{w}^{T} (\phi\left(W_{n-1} \phi\left(W_{n-2}\left(\ldots \phi\left(W_{1} \mathbf{x}\right)\right)\right)\right)\right|-\rho$, $\beta_{p,q}(W) = \prod_{k=1}^{n}\|W_{k}\|_{p,q}\leq \beta$ and $\phi$ is ReLU activation function. Also, $\mathbf{w}$ is an $H$ dimensional vector. We prove the bound by induction
\begin{align*}
    \hspace{-0.5 cm}R(\mathcal{F}_{\beta_{p,q}\leq\beta}^{n,H}) &=\mathbb{E}_{\rrv}\left[\frac{1}{m} \sup _{f \in \mathcal{F}_{\beta_{p, q}  \leq \beta}^{n, H}}\sup_{\rho}\left|\sum_{i=1}^{m} \rrvi{i}(| f(\mathbf{x}_{i})| -\rho) \right|\right] \nonumber \\
    R(\mathcal{F}_{\beta_{p,q}\leq\beta}^{n,H})&\leq \mathbb{E}_{\rrv}\left[\frac{1}{m} \sup _{f \in \mathcal{F}_{\beta_{p, q} \leq \beta}^{n, H}}\sup_{\rho}\left|\sum_{i=1}^{m} \rrvi{i}| f(\mathbf{x}_{i}) |\right|\right] \\&\hspace{1 cm}+  \mathbb{E}_{\rrv}\left[\frac{1}{m} \sup _{f \in \mathcal{F}_{\beta_{p, q} \leq \beta}^{n, H}}\sup_{\rho}\left|\sum_{i=1}^{m} \rrvi{i} \rho \right|\right] \nonumber \\
    R(\mathcal{F}_{\beta_{p,q}\leq\beta}^{n,H})&= \mathbb{E}_{\rrv}\left[\frac{1}{m} \sup _{f \in \mathcal{F}_{\beta_{p, q} \leq \beta}^{n, H}}\sup_{\rho}\left|\sum_{i=1}^{m} \rrvi{i}| f(\mathbf{x}_{i}) |\right|\right] \\&\hspace{2 cm }+  \frac{1}{m}\urho\mathbb{E}_{\rrv}\left[\left|\sum_{i=1}^{m} \rrvi{i}\right|\right] \nonumber \\
    &\leq \mathbb{E}_{\rrv}\left[\frac{1}{m} \sup _{f \in \mathcal{F}_{\beta_{p, q} \leq \beta}^{n, H}}\sup_{\rho}\left|\sum_{i=1}^{m} \rrvi{i}| f(\mathbf{x}_{i}) |\right|\right] \\&\hspace{2cm}+  \frac{1}{m}\urho\sqrt{m}(1) \nonumber \\
    &=\mathbb{E}_{\rrv}\left[\frac{1}{m} \sup _{f \in \mathcal{F}_{\beta_{p, q} \leq \beta}^{n, H}}\left|\sum_{i=1}^{m} \rrvi{i}| f(\mathbf{x}_{i}) |\right|\right]+  \frac{\urho}{\sqrt{m}}
\end{align*}
Let, $\mathcal{R}_{rec}$ be defined as,
\begin{equation}
\mathcal{R}_{rec} = \mathbb{E}_{\rrv}\left[\frac{1}{m} \sup _{f \in \mathcal{N}^{n, H}} \frac{\beta}{\beta_{p, q}(f)}\left|\sum_{i=1}^{m} \rrvi{i} |f(\mathbf{x}_{i})|\right|\right] \nonumber     
\end{equation}
Also, $\mathcal{R}_{rec}= R(\mathcal{N}_{\beta_{p,q}\leq\beta}^{n,H})$ where $\mathcal{N}(\mathbf{x}) = \left|\mathbf{w}^{T} (\phi\left(W_{n-1} \phi\left(W_{n-2}\left(\ldots \phi\left(W_{1} \mathbf{x}\right)\right)\right)\right)\right|$ and $\phi$ is ReLU activation function.
    \begin{align*}
  &\mathcal{R}_{rec}=\mathbb{E}_{\rrv}\left[\frac{1}{m} \sup _{f \in \mathcal{N}^{n, H}} \frac{\beta}{\beta_{p, q}(f)}\left|\sum_{i=1}^{m} \rrvi{i} |f(\mathbf{x}_{i})|\right|\right] \nonumber\\
    &=\mathbb{E}_{\rrv}\Bigg[\frac{1}{m} \sup _{g \in \mathcal{N}^{n-1, H, H}} \\&\hspace{2cm}\sup _{\mathbf{w}} \frac{\beta}{\beta_{p, q}(g)\|\mathbf{w}\|_{p}}\left|\sum_{i=1}^{m} \rrvi{i} |\mathbf{w}^{T}[g(\mathbf{x}_{i})]_{+}|\right|\Bigg] \nonumber \\
    &=\mathbb{E}_{\rrv}\Bigg[\frac{1}{m} \sup _{g \in \mathcal{N}^{n-1, H, H}} \\ &\hspace{2 cm}\sup _{\mathbf{w}} \frac{\beta}{\beta_{p, q}(g)\|\mathbf{w}\|_{p}}\left|\sum_{i=1}^{m} \rrvi{i} \|\mathbf{w}\|_{p}\|[g(\mathbf{x}_{i})]_{+}\|_{\pc}\right|\Bigg]\nonumber \\
    &=\mathbb{E}_{\rrv}\left[\frac{1}{m} \sup _{g \in \mathcal{N}^{n-1, H, H}} \frac{\beta}{\beta_{p, q}(g)}\left|\sum_{i=1}^{m} \rrvi{i}\|\left[g(\mathbf{x}_{i})\right]_{+}\|_{\pc}\right|\right]\nonumber \end{align*}
    \begin{align*}
    &=\mathbb{E}_{\rrv}\Bigg[\frac{1}{m} \sup _{h \in \mathcal{N}^{n-2, H, H}} \frac{\beta}{\beta_{p, q}(h)} \\ &\hspace{2 cm}\sup _{W} \frac{1}{\|W\|_{p, q}}\left|\sum_{i=1}^{m} \rrvi{i}\left\|[W\left[h(\mathbf{x}_{i})\right]_{+}\right]_{+}\|_{\pc}\right|\Bigg] 
    \nonumber \\
    &= \mathbb{E}_{\rrv}\Bigg[\frac{1}{m} \sup _{h \in \mathcal{N}^{n-2, H, H}} \frac{\beta}{\beta_{p, q}(h)} \\ &\hspace{2 cm}\sup _{W} \frac{1}{\|W\|_{p, q}}\left\|\sum_{i=1}^{m} \rrvi{i}\left\|[W\left[h(\mathbf{x}_{i})]_{+}\right]_{+}\right\|_{\pc}\right\|_{\pc}\Bigg] 
    \nonumber 
    \end{align*}
We use Lemma \ref{lemma:Wtow} to obtain the following result,
    \begin{align*}
    &R(\mathcal{N}_{\beta_{p,q}\leq\beta}^{n,H})= H^{\left[\frac{1}{\pc}-\frac{1}{q}\right]_{+}}\mathbb{E}_{\rrv}\Bigg[\frac{1}{m} \sup _{h \in \mathcal{N}^{n-2, H, H}} \\ &\hspace{1.5 cm}\frac{\beta}{\beta_{p, q}(h)}\sup _{\mathbf{w}} \frac{1}{\|\mathbf{w}\|_{p}}\left|\sum_{i=1}^{m} \rrvi{i}\left\|\left[\mathbf{w}^{\top}\left[h\left(\mathbf{x}_{i}\right)\right]_{+}\right]_{+}\right\|_{\pc}\right| \Bigg]
    \nonumber\\
    &\hspace{-0.2 cm}= H^{\left[\frac{1}{\pc}-\frac{1}{q}\right]_{+}}\mathbb{E}_{\rrv}\Bigg[\frac{1}{m} \sup _{h \in \mathcal{N}^{n-1, H, H}} \\ &\hspace{4 cm}\frac{\beta}{\beta_{p, q}(g)}\left|\sum_{i=1}^{m} \rrvi{i}\left\|\left[g(\mathbf{x}_{i})\right]_{+}\right\|_{\pc}\right| \Bigg] \nonumber \\
    &\hspace{-0.2 cm}\leq H^{\left[\frac{1}{\pc}-\frac{1}{q}\right]_{+}}\mathbb{E}_{\rrv}\left[\frac{1}{m} \sup _{h \in \mathcal{N}^{n-1, H, H}} \frac{\beta}{\beta_{p, q}(g)}\left|\sum_{i=1}^{m} \rrvi{i}\left|\left[g(\mathbf{x}_{i})\right]_{+}\right|\right| \right]  \nonumber \\
    &\hspace{-0.2 cm}\leq H^{\left[\frac{1}{\pc}-\frac{1}{q}\right]_{+}}\mathbb{E}_{\rrv}\left[\frac{1}{m} \sup _{h \in \mathcal{N}^{n-1, H, H}} \frac{\beta}{\beta_{p, q}(g)}\left|\sum_{i=1}^{m} \rrvi{i}\left[\left|g(\mathbf{x}_{i})\right|\right]_{+}\right| \right] \nonumber 
\end{align*}
We use Lemma 16 (Contraction Lemma) \cite{neyshabur2015norm} result directly to obtain the following result,
\begin{align*}
    R(\mathcal{N}_{\beta_{p,q}\leq\beta}^{n,H})&\leq 2(1)H^{\left[\frac{1}{\pc}-\frac{1}{q}\right]_{+}}\mathbb{E}_{\rrv}\Bigg[\frac{1}{m} \\ &\hspace{2 cm}\sup _{h \in \mathcal{N}^{n-1, H, H}} \frac{\beta}{\beta_{p, q}(g)}\left|\sum_{i=1}^{m} \rrvi{i}\left|g(\mathbf{x}_{i})\right|\right| \Bigg] \nonumber\\
    R(\mathcal{N}_{\beta_{p,q}\leq\beta}^{n,H})&\leq 2H^{\left[\frac{1}{\pc}-\frac{1}{q}\right]_{+}}R\left(\mathcal{N}_{\beta_{p, q} \leq \beta}^{n-1, H}\right)\nonumber
    \end{align*}
We now use the recurrence relationship and Rademachar complexity obtained from Theorem \ref{thm-one_layer} to get,
    \begin{align}
    R(\mathcal{F}_{\beta_{p,q}\leq\beta}^{n,H})&\leq \left(\frac{\beta}{\sqrt{m}}\max_{i}\|\mathbf{x}_{i}\|_{\pc}\right)\left(2H^{\left[\frac{1}{\pc}-\frac{1}{q}\right]_{+}}\right)^{n-1} +\frac{\urho}{\sqrt{m}}
    \nonumber
\end{align}
The Lipschitz constant for double sigmoid loss with cost of rejection $d_{r}$ can be computed as 
\begin{align}
    L = \sup &|2d_{r}\sigma(z-\rho)\left(1-\sigma(z-\rho)\right)\nonumber\\&+2(1-d_{r})\sigma(z+\rho)\left(1-\sigma(z+\rho)\right)|
    \label{eq:lipshitz}
\end{align}
The maximum value of the product $\sigma(z-\rho)\left(1-\sigma(z-\rho)\right)$ and $\sigma(z+\rho)\left(1-\sigma(z+\rho)\right)$ is at $z =\rho$ and $z=-\rho$ respectively. However, the maximum value of the equation \ref{eq:lipshitz} is achieved at $z = 0$ and $\rho = 0$.
Thus,
\begin{align*}
  &L = 2d_{r}\sigma(-\rho)(1-\sigma(-\rho)) + 2(1-d_{r})\sigma(\rho)(1-\sigma(\rho))\\
  &L = 2d_{r}\sigma(-\rho)\sigma(\rho) + 2(1-d_{r})\sigma(\rho)\sigma(-\rho) \\
  &L = 2\sigma(\rho)\sigma(-\rho)\\
  &L = 2\sigma(\rho)\sigma(-\rho)
\end{align*}  when $\rho = 0$, we get $L=0.5$. 
We now use the above result to get the generalization bound where the lipschitz constant $L$ for double sigmoid loss would be
$L = 0.5$.

\begin{align*}R_{ds}(f, \rho) &\leq \Rhds{f}{\rho}\\&+\left(\frac{2\beta}{\sqrt{m}}\max_{i}\|\mathbf{x}_{i}\|_{\pc}\right)\left(2H^{\left[\frac{1}{\pc}-\frac{1}{q}\right]_{+}}\right)^{n-1} \\ &\hspace{0.0cm}+\frac{\urho}{\sqrt{m}}+ \sqrt{\frac{8\ln \left(\frac{4}{\delta}\right)}{m}} + \sqrt{\frac{2\ln \left(\frac{2}{\delta}\right)}{m}} \nonumber
\end{align*}
\end{proof}
\begin{lemma}
Let $\mathcal{Y} \subseteq \mathbb{R},$ and $\operatorname{let} \mathcal{F} \subseteq[a, b]^{\mathcal{X}}$ for some $a \leq b .$ Let $\ell: \mathcal{Y} \times[a, b] \rightarrow[0, B]$ be such that $\ell(y, \hat{y})$
is $L$ -Lipschitz in its second argument for some $L>0 .$ Let $D$ be any probability distribution on $\mathcal{X} \times \mathcal{Y},$ with marginal $\mu$ on $\mathcal{X}$. If $f_{S}$ is selected from $\mathcal{F},$ then for any $0<\delta \leq 1,$ with probability at least $1-\delta$ (over $\left.S \sim D^{m}\right)$
\begin{align*}
        R_{ds}(f, \rho) \leq& \Rhds{f}{\rho}+2 L R(\mathcal{F})+\\&2(b-a)\sqrt{\frac{\ln
        \left(\frac{4}{\delta}\right)}{2 m}}+\left(b-a\right)\sqrt{\frac{ln\left(\frac{2}{\delta}\right)}{m}} \nonumber
\end{align*}
\label{lemma:mcdarmid}
\end{lemma}
\begin{proof}
First we define 
\begin{equation}
    \hat{R}_{m}(\mathcal{F}) = \mathbb{E}_{\{\rrv\in{\pm1}\}}\left[\sup_{f\in\mathcal{F}}\frac{1}{m}\sum_{i=1}^{m}\rrvi{i}f(\mathbf{x}_{i})\right] \nonumber
\end{equation}
where $\rrv$ is the Rademachar variable and $R_{m}(\mathcal{F})$ is defined as expectation over data samples of size m obtained in an i.i.d fashion from probability distribution $\mu$ i.e.
\begin{equation}
    R_{m}(\mathcal{F}) = \mathbb{E}_{\mathbf{x}^{m}\sim\mu^{m}}\left[\hat{R}_{m}(\mathcal{F})\right] \nonumber
\end{equation}
We directly use the result from the \cite{bartlett2002rademacher} and using the results directly with probability atleast $1-\delta$, we bound the generalization error as,
\begin{equation}R_{ds}(f, \rho) \leq \Rhds{f}{\rho}+2LR_{m}({\mathcal{F}})+2(b-a) \sqrt{\frac{\ln \left(\frac{2}{\delta}\right)}{2 m}}  \label{eq:bm}
\end{equation}
Now for any set $\mathcal{S}=\{\mathbf{x}_{1},\mathbf{x}_{2}....,\mathbf{x}_{m}\}$, and a function $\phi: \mathcal{X}^{m}\rightarrow\mathbb{R}$ such that $\phi(\mathbf{x}_{1},\mathbf{x}_{2},...\mathbf{x}_{m}) = \hat{R}_{m}(\mathcal{F})$. Hence, $R_{m}(\mathcal{F}) = \mathbf{E}_{\mathbf{x}^{m} \sim \mu^{m}}\left[\phi\left(\mathbf{x}_{1}, \ldots, \mathbf{x}_{m}\right)\right]$

Then, for any $j\in[m]$, and any $\mathbf{x}_{1}, \ldots, \mathbf{x}_{m}, \mathbf{x}_{j}^{\prime} \in \mathcal{X}$
\begin{align*}
    &\left|\phi\left(\mathbf{x}_{1}, \ldots, \mathbf{x}_{j}, \ldots, \mathbf{x}_{m}\right)-\phi\left(\mathbf{x}_{1}, \ldots, \mathbf{x}_{j}^{\prime}, \ldots, \mathbf{x}_{m}\right)\right| \nonumber \\
    &=\hat{R}_{m}(\mathcal{F})-R_{\left(\mathbf{x}_{1}, \ldots, \mathbf{x}_{j}^{\prime}, \ldots, \mathbf{x}_{m}\right)}(\mathcal{F}) \nonumber \\
    &=\mathbf{E}_{\rrv \in\{\pm 1\}^{m}}\Bigg[\sup _{f \in \mathcal{F}} \frac{1}{m} \sum_{i=1}^{m} \rrvi{i} f\left(\mathbf{x}_{i}\right)\\&\hspace{2 cm }-\sup _{f \in \mathcal{F}}\left(\frac{1}{m} \sum_{i \neq j} \rrvi{i} f\left(\mathbf{x}_{i}\right)+\frac{1}{m} \rrvi{j} f\left(\mathbf{x}_{j}^{\prime}\right)\right)\Bigg]\\
    &\leq \frac{b-a}{m} \nonumber
\end{align*}
Thus by McDarmid's inequality we get,
\begin{equation}
  \mathbf{P} \left[\hat{R}_{m}(\mathcal{F})-R_{m}(\mathcal{F}) \geq \epsilon\right] \leq e^{-2 m \epsilon^{2} /(b-a)^{2}} 
\end{equation}
Now with probability atleast $1-\frac{\delta}{2}$ we get,
\begin{equation}
    \hat{R}_{m}(\mathcal{F})-R_{m}(\mathcal{F}) \leq 2(b-a)\sqrt{\frac{\ln{\frac{2}{\delta}}}{2m}}
    \label{eq:rel-rad}
\end{equation}
We also know that, a relationship exits between 
Using the combination of eqn. (\ref{eq:rel-rad}) and eqn. (\ref{eq:bm}) each holding with a probability of atleast $1-\frac{\delta}{2}$, we get with a probability atleast $1-\delta$,
\begin{align*}
    R_{ds}(f, \rho) &\leq \Rhds{f}{\rho}+2 L \hat{R}_{m}(\mathcal{F})\\&+2(b-a)\sqrt{\frac{\ln \left(\frac{4}{\delta}\right)}{2 m}}+\left(b-a\right)\sqrt{\frac{\ln\left(\frac{2}{\delta}\right)}{m}} \nonumber
\end{align*}
\end{proof}
\begin{theorem}
The rademachar complexity for RISAN with a single layer $(n=1)$, is bounded as
\begin{equation}
    R(\mathcal{F}_{\|\mathbf{w}\|_{p}\leq\beta}^{1}) \leq \frac{\beta}{\sqrt{m}}\max_{i}\|\mathbf{x}_{i}\|_{\pc} \nonumber
\end{equation}
     \label{thm-one_layer}
\end{theorem}

\begin{proof}
For a network with a single layer, it is important to notice that $\beta_{p,q}(\mathbf{w})=\|\mathbf{w}\|_{p}$.
\begin{align}
    R(\mathcal{F}_{\|\mathbf{w}\|_{p}\leq\beta}^{1}) &\leq \mathbb{E}_{\rrv}\left[\frac{1}{m} \sup _{\|\mathbf{w}\|_{p} \leq \beta}\sup_{\rho}\left|\sum_{i=1}^{m} \rrvi{i}| \mathbf{w}^{T}\mathbf{x}_{i}| \right|\right] \nonumber \\
    R(\mathcal{F}_{\|\mathbf{w}\|_{p}\leq\beta}^{1}) &\leq \mathbb{E}_{\rrv}\left[\frac{1}{m} \sup _{\|\mathbf{w}\|_{p} \leq \beta}\sup_{\rho}\left|\sum_{i=1}^{m} \rrvi{i}| \mathbf{w}^{T}\mathbf{x}_{i}| \right|\right] \nonumber \\
    R(\mathcal{F}_{\|\mathbf{w}\|_{p}\leq\beta}^{1}) &\leq \mathbb{E}_{\rrv}\left[\frac{1}{m} \sup _{\|\mathbf{w}\|_{p} \leq \beta}\sup_{\rho}\left|\sum_{i=1}^{m} \rrvi{i}| \mathbf{w}^{T}\mathbf{x}_{i}| \right|\right] \nonumber \\
    R(\mathcal{F}_{\|\mathbf{w}\|_{p}\leq\beta}^{1}) &\leq \mathbb{E}_{\rrv}\left[\frac{1}{m} \sup _{\|\mathbf{w}\|_{p} \leq \beta}\sup_{\rho}\left|\sum_{i=1}^{m} \rrvi{i}\| \mathbf{w}\|_{p}\|\mathbf{x}_{i}\|_{\pc} \right|\right] \nonumber \\
    R(\mathcal{F}_{\|\mathbf{w}\|_{p}\leq\beta}^{1}) &\leq \beta\mathbb{E}_{\rrv}\left[\frac{1}{m}\left|\sum_{i=1}^{m} \rrvi{i}\|\mathbf{x}_{i}\|_{\pc} \right|\right]  \nonumber \\
    R(\mathcal{F}_{\|\mathbf{w}\|_{p}\leq\beta}^{1}) &\leq \frac{\beta}{m}\left(\sum_{i=1}^{m}\|\mathbf{x}_{i}\|_{\pc}^2 \right)^{\frac{1}{2}} \nonumber \\
    R(\mathcal{F}_{\|\mathbf{w}\|_{p}\leq\beta}^{1}) &\leq \frac{\beta}{m}\left(m \max_{i}\|\mathbf{x}_{i}\|_{\pc}^2 \right)^{\frac{1}{2}} \nonumber \\
     &= \frac{\beta}{\sqrt{m}} \max_{i}\|\mathbf{x}_{i}\|_{\pc}  \nonumber
\end{align}
\end{proof}

\begin{lemma}
For any $p$, $q \geq 1$, $n \geq 2$, $\lambda \in\{\pm 1\}^{m}$  and  $f \in \mathcal{N}^{n, H, H}$
\begin{align*}
&\sup _{W} \frac{1}{\|W\|_{p, q}}\left\|\sum_{i=1}^{m} \rrvi{i}\left\|\left[W\left[f\left(\mathbf{x}_{i}\right)\right]_{+}\right]_{+}\right\|_{\pc}\right\|_{\pc}=\\ &\hspace{0.5cm}H^{\left[\frac{1}{\pc}-\frac{1}{q}\right]+} \sup _{\mathbf{w}} \frac{1}{\|\mathbf{w}\|_{p}}\left|\sum_{i=1}^{m} \rrvi{i}\left\|\left[\mathbf{w}^{\top}\left[f\left(\mathbf{x}_{i}\right)\right]_{+}\right]_{+}\right\|_{\pc}\right| \nonumber
\end{align*}
where $n$ is the depth of the network, $H$ is the height of the layer and $H$ is the no. of outputs.
\label{lemma:Wtow}
\end{lemma}

\begin{proof}
\begin{equation}
  g(\mathbf{w}) =  \left| \sum_{i=1}^{m} \rrvi{i} \|\mathbf{w}^{\top}\left[f(\mathbf{x}_{i}\right)]_{+}\|_{\pc}\right|\nonumber
\end{equation}
We define $\mathbf{w}^{*}$ as
\begin{equation}
    \mathbf{w}^{*} = \arg \max _{\mathbf{w}} \frac{g(\mathbf{w})}{\|\mathbf{w}\|_{p}} \nonumber
\end{equation} 
Thus,
\begin{align}
    g(V_{i}) &= \left| \sum_{i=1}^{m} \rrvi{i} \|V_{i}^{\top}\left[f(\mathbf{x}_{i}\right)]_{+}\|_{\pc}\right|\nonumber 
\end{align}
where $V_{i}$ is row of any matrix V.
Now we know that,
\begin{align}
    \frac{g({\mathbf{w}}^{*})}{\|{\mathbf{w}}^{*}\|_{p}} &\geq \frac{g(V_{i})}{\|V_{i}\|_{p}} \nonumber\\
    \bigg(\frac{g({\mathbf{w}}^{*})}{\|{\mathbf{w}}^{*}\|_{p}}\bigg)^{\pc} &\geq \bigg(\frac{g(V_{i})}{\|V_{i}\|_{p}}\bigg)^{\pc} \nonumber\\
    \bigg(\frac{g({\mathbf{w}}^{*}){\|V_{i}\|_{p}}}{\|{\mathbf{w}}^{*}\|_{p}}\bigg)^{\pc} &\geq \big(g(V_{i})\big)^{\pc} \nonumber\\
    \sum_{i=1}^{H}\bigg(\frac{g({\mathbf{w}}^{*}){\|V_{i}\|_{p}}}{\|{\mathbf{w}}^{*}\|_{p}}\bigg)^{\pc} &\geq \sum_{i=1}^{H}\big(g(V_{i})\big)^{\pc} \nonumber\\
    \bigg(\frac{g({\mathbf{w}}^{*})}{\|{\mathbf{w}}^{*}\|_{p}}\bigg)^{\pc}\sum_{i=1}^{H}\|V_{i}\|_{p}^{\pc} &\geq \sum_{i=1}^{H}\big(g(V_{i})\big)^{\pc} \nonumber\\
    \frac{g({\mathbf{w}}^{*})}{\|{\mathbf{w}}^{*}\|_{p}}\bigg(\sum_{i=1}^{H}\|V_{i}\|_{p}^{\pc}\bigg)^{\frac{1}{\pc}} &\geq \bigg(\sum_{i=1}^{H}\big(g(V_{i})\big)^{\pc}\bigg)^{\frac{1}{\pc}} \nonumber\\
    \frac{g({\mathbf{w}}^{*})}{\|{\mathbf{w}}^{*}\|_{p}} &\geq \frac{\|g(V)\|_{\pc}}{\|V\|_{p,\pc}} \label{eqn: Wtow}
\end{align}
We have 2 cases now,
$q>\pc$ and $q<\pc$. If $q<\pc$ $\|V\|_{p,\pc}\leq\|V\|_{p,q}$ and $H^{[\frac{1}{\pc}-\frac{1}{q}]_{+}}=1$.
Thus,
\begin{equation}
    \frac{g({\mathbf{w}}^{*})}{\|{\mathbf{w}}^{*}\|_{p}} \geq \frac{\|g(V)\|_{\pc}}{\|V\|_{p,q}} \nonumber
\end{equation}
We also know that $\|V\|_{p,\pc}\leq H^{[\frac{1}{\pc}-\frac{1}{q}]}\|V\|_{p,q}$ and thus,
\begin{equation}
    \frac{\|g(V)\|_{\pc}}{\|V\|_{p,q}} \leq H^{[\frac{1}{\pc}-\frac{1}{q}]}\frac{\|g(V)\|_{\pc}}{\|V\|_{p,\pc}} \nonumber
\end{equation}
And from eqn.(\ref{eqn: Wtow}) we get,
\begin{align}
    \frac{g({\mathbf{w}}^{*})}{\|{\mathbf{w}}^{*}\|_{p}} \geq \frac{\|g(V)\|_{\pc}}{\|V\|_{p,\pc}} &\geq \frac{\|g(V)\|_{\pc}}{H^{[\frac{1}{\pc}-\frac{1}{q}]}\|V\|_{p,q}} \nonumber \\
    H^{[\frac{1}{\pc}-\frac{1}{q}]}\frac{g({\mathbf{w}}^{*})}{\|{\mathbf{w}}^{*}\|_{p}} &\geq \frac{\|g(V)\|_{\pc}}{\|V\|_{p,q}} \nonumber 
\end{align}
The LHS of the lemma is greater than RHS is true for any given vector $\mathbf{w}$, not $\mathbf{w}^{*}$ in the RHS. Also, the equality exists when $W$ matrix contains $\mathbf{w}^{*}$ as all of its rows.
\end{proof}

\begin{table*}[t]
		
		\begin{tabularx}{\textwidth}{|  >{\hsize=2.1cm}Y  | >{\hsize=2.3cm}Y | Y | Y | Y | Y |}
			\hline
			\multirow{2}{*}{} & \multirow{2}{*}{\textbf{RISAN}} & \multirow{2}{*}{\textbf{RISAN-NA}} & \multirow{2}{*}{\textbf{SNN}} & \multirow{2}{*}{\textbf{SNN-NA}} & \multirow{2}{*}{\textbf{DAC}}
			\\ 
			\multirow{2}{*}{} & \multirow{2}{*}{} & \multirow{2}{*}{} & \multirow{1}{*}{} & \multirow{2}{*}{} & \multirow{1}{*}{} \\
			\multirow{1}{*}{}& \multirow{1}{*}{Double Sigmoid} & \multirow{1}{*}{Double Sigmoid} & \multirow{1}{*}{Selective Loss} & \multirow{1}{*}{Selective Loss} & \multirow{1}{*}{DAC Loss}
			\\ 
			\multirow{1}{*}{}& \multirow{1}{*}{+ Cross Entropy} & \multirow{1}{*}{} & \multirow{1}{*}{+ Cross Entropy} & \multirow{1}{*}{} & \multirow{1}{*}{}
			\\ \hline
			\multirow{2}{*}{\textbf{Architecture}} & \multirow{2}{*}{} & \multirow{1}{*}{} & \multirow{2}{*}{} & \multirow{1}{*}{} & \multirow{2}{*}{} \\
			\multirow{2}{*}{} & \multirow{2}{*}{} & \multirow{1}{*}{} & \multirow{2}{*}{} & \multirow{1}{*}{} & \multirow{2}{*}{} \\\hline
			\multirow{1}{*}{FC Layers} & \multirow{1}{*}{} & \multirow{1}{*}{} & \multirow{1}{*}{} & \multirow{1}{*}{} & \multirow{1}{*}{}\\
			\multirow{1}{*}{Prediction head} & \multirow{1}{*}{512,256,128} & \multirow{1}{*}{512,256,128} & \multirow{1}{*}{512,256,128} & \multirow{1}{*}{512,256,128} & \multirow{1}{*}{512,256,128} 
			\\ \multirow{1}{*}{Rejection head/}& \multirow{2}{*}{512,256,64} & \multirow{2}{*}{512,256,64} & \multirow{2}{*}{512,256,64} & \multirow{2}{*}{512,256,64} & \multirow{2}{*}{} 
			\\ \multirow{1}{*}{Selective head} & \multirow{1}{*}{} & \multirow{1}{*}{} & \multirow{1}{*}{} & \multirow{1}{*}{} & \multirow{1}{*}{}\\
			\hline
			\multirow{1}{*}{Weight Decay} & \multirow{1}{*}{} & \multirow{1}{*}{} & \multirow{1}{*}{} & \multirow{1}{*}{} & \multirow{1}{*}{}
			\\ \multirow{1}{*}{CNN}& \multirow{1}{*}{1e-4} & \multirow{1}{*}{1e-4} & \multirow{1}{*}{1e-4} & \multirow{1}{*}{1e-4} & \multirow{1}{*}{1e-4} 
			\\ \multirow{1}{*}{FC} & \multirow{1}{*}{1e-7} & \multirow{1}{*}{1e-7} & \multirow{1}{*}{1e-7} & \multirow{1}{*}{1e-7} & \multirow{1}{*}{1e-7}
			\\ \hline
			\multirow{2}{*}{\textbf{Datasets}} & \multirow{1}{*}{} & \multirow{1}{*}{} & \multirow{1}{*}{} & \multirow{1}{*}{} & \multirow{1}{*}{}
			\\ \multirow{2}{*}{} & \multirow{2}{*}{} & \multirow{1}{*}{} & \multirow{2}{*}{} & \multirow{1}{*}{} & \multirow{2}{*}{} \\\hline
			\multirow{1}{*}{}
			 & \multirow{1}{*}{Figure~\ref{fig:DNN-2}}
			 & \multirow{1}{*}{Figure~\ref{fig:DNN-na-1}}  & \multirow{1}{*}{Figure~\ref{fig:DNN-2}}  & \multirow{1}{*}{Figure~\ref{fig:DNN-na-1}}   & \multirow{1}{*}{Figure~\ref{fig:DNN-na-1}}  \\
			 \multirow{2}{*}{Cats vs Dogs} & $\alpha$ = 0.9 $\gamma = 1e-3$ & $\gamma = 1e-3$ & $\alpha$=0.5 & & (no rejection head)\\ 
			  \multirow{1}{*}{}& 250 epochs & 250 epochs & 250 epochs & 250 epochs& 250 epochs
			 \\\hline
			 \multirow{4}{*}{}
			 & \multirow{1}{*}{Figure~\ref{fig:DNN-2}} & \multirow{1}{*}{Figure~\ref{fig:DNN-na-1}}  & \multirow{1}{*}{Figure~\ref{fig:DNN-2}}  & \multirow{1}{*}{Figure~\ref{fig:DNN-na-1}}   & \multirow{1}{*}{Figure~\ref{fig:DNN-na-1}}  \\
			 CIFAR & $\alpha$ = 0.7 $\gamma = 1e-3$ & $\gamma = 1e-3$ & $\alpha$=0.5 & & (no rejection head)\\
			 \multirow{4}{*}{}& 250 epochs & 250 epochs & 250 epochs & 250 epochs& 250 epochs\\ \hline\multirow{3}{*}{}
			 & \multirow{1}{*}{Figure~\ref{fig:DNN-2}} & \multirow{1}{*}{Figure~\ref{fig:DNN-na-1}}  & \multirow{1}{*}{Figure~\ref{fig:DNN-2}}  & \multirow{1}{*}{Figure~\ref{fig:DNN-na-1}}   & \multirow{1}{*}{Figure~\ref{fig:DNN-na-1}}  \\
			 MNIST & $\alpha$ = 0.7 $\gamma = 1e-3$ & $\gamma = 1e-3$ & $\alpha$=0.5 & & (no rejection head)\\\multirow{1}{*}{}& 150 epochs & 150 epochs & 150 epochs & 150 epochs& 150 epochs\\  \hline
			 \multirow{4}{*}{}
			 & \multirow{1}{*}{Figure~\ref{fig:DNN-2}} & \multirow{1}{*}{Figure~\ref{fig:DNN-na-1}}  & \multirow{1}{*}{Figure~\ref{fig:DNN-2}}  & \multirow{1}{*}{Figure~\ref{fig:DNN-na-1}}   & \multirow{1}{*}{Figure~\ref{fig:DNN-na-1}}  \\
			 CBIS-DDSM & $\alpha$ = 0.7   
			 $\gamma = 1.0$ & $\gamma = 1e-3$ & $\alpha$=0.5 & & (no rejection head)\\\multirow{1}{*}{}& 250 epochs & 250 epochs & 250 epochs & 250 epochs& 250 epochs\\  \hline
		\end{tabularx}
		\caption{Architecture and Hyperparameters for Large datasets}
		\label{Table:tbl}
	\end{table*}

	\subsection{Architecture Details and Hyperparameter Selection}
 	\label{sec:arc-details-hyp-selection}
\paragraph{Small Datasets Experiments:}
The experiments with regular small dimensional data are conducted with the network architecture shown in Figure ~\ref{fig:DNN-1}. We use 3 fully connected layers in the main body block of the network architecture with batch normalization \citep{ioffe2015batch} and dropout \citep{srivastava2014dropout} at each layer. Each layer uses ReLU (Rectified Linear Units) as the activation function with 64 neurons in each layer. We further use Adagrad (adaptive gradient) optimizer with a learning rate of $1\mathrm{e}{-3}$ and run it for 100 epochs. We fix the batch size as 32. We use a $\gamma$ value of 2 in the double sigmoid loss function for Ionosphere dataset whereas a value of 1 for ILPD dataset. For both SDR-SVM and DH-SVM we use a Gaussian kernel. We select the best values of regularization parameter $\lambda$ and kernel parameter $\gamma$ using 10-fold cross validation. We also use $\mu=1$ for SDR-SVM.

\paragraph{Large Datasets Experiments:}
For phishing dataset, we use RISAN with input dependent rejection.
However, the rejection head is dependant on the input, hence it gets input not from a constant valued neuron but the fully connected (FC) layers. We used 4 FC layers with 64 neurons each, and followed each layer with dropout and batch normalization layers. The same architecture is used for experiments with SNN-NA and DAC with phishing dataset. 
The phishing dataset being a small dimensional dataset, the auxiliary loss becomes redundant and hence we remove the auxiliary loss for this experiment for SNN.

In the CNN based experiments, we used and followed the architecture and hyperparameters similar to the ones used in \cite{geifman2019selectivenet}.
The VGG-16 architecture from \cite{simonyan2014very} was  optimized for the small datasets and image sizes as suggested in \cite{liu2015very} with following alterations: 
 (i) used only one fully connected layer with 512 neurons (the original VGG-16 has two fully connected layers of 4096 neurons). (ii) added batch normalization \cite{ioffe2015batch} (iii)  added dropout \cite{srivastava2014dropout}.
 Also, the standard data augmentation consisting of horizontal flips, vertical and horizontal shifts, and rotations were included. The network was optimized using stochastic gradient descent (SGD) with a momentum of 0.9, an initial learning rate of 0.1, and a weight decay of 5e-4. The learning rate was reduced by 0.5 every 25 epochs.
 
We made further amendments to it by incorporating a separate stack of fully connected layers for each head. While the main body block of across all algorithms is the VGG-16 architecture. We added individual hidden layers to both prediction head and rejection head. We used three fully connected layers of sizes 512, 256 and 128, followed by a single neuron (prediction head). We also used additional three fully connected layers of size 512, 256 and 64 neurons followed by a single neuron for the rejection head. Moreover, for the rejection head, ReLU activation was used to ensure a positive rejection region parameter. A separate weight decay of $1\mathrm{e}{-4}$ for CNN layers and $1\mathrm{e}{-7}$ for FC layers. The data for individual CNN based datasets and algorithm is provided in Table \ref{Table:tbl}.
The learning rate scheduler was used which reduces learning rate by 0.5 once the validation loss stagnates. We utilize the same learning rate scheduler across all algorithms and datasets. 
\begin{figure}
    \centering
    \includegraphics[width=0.4 \textwidth]{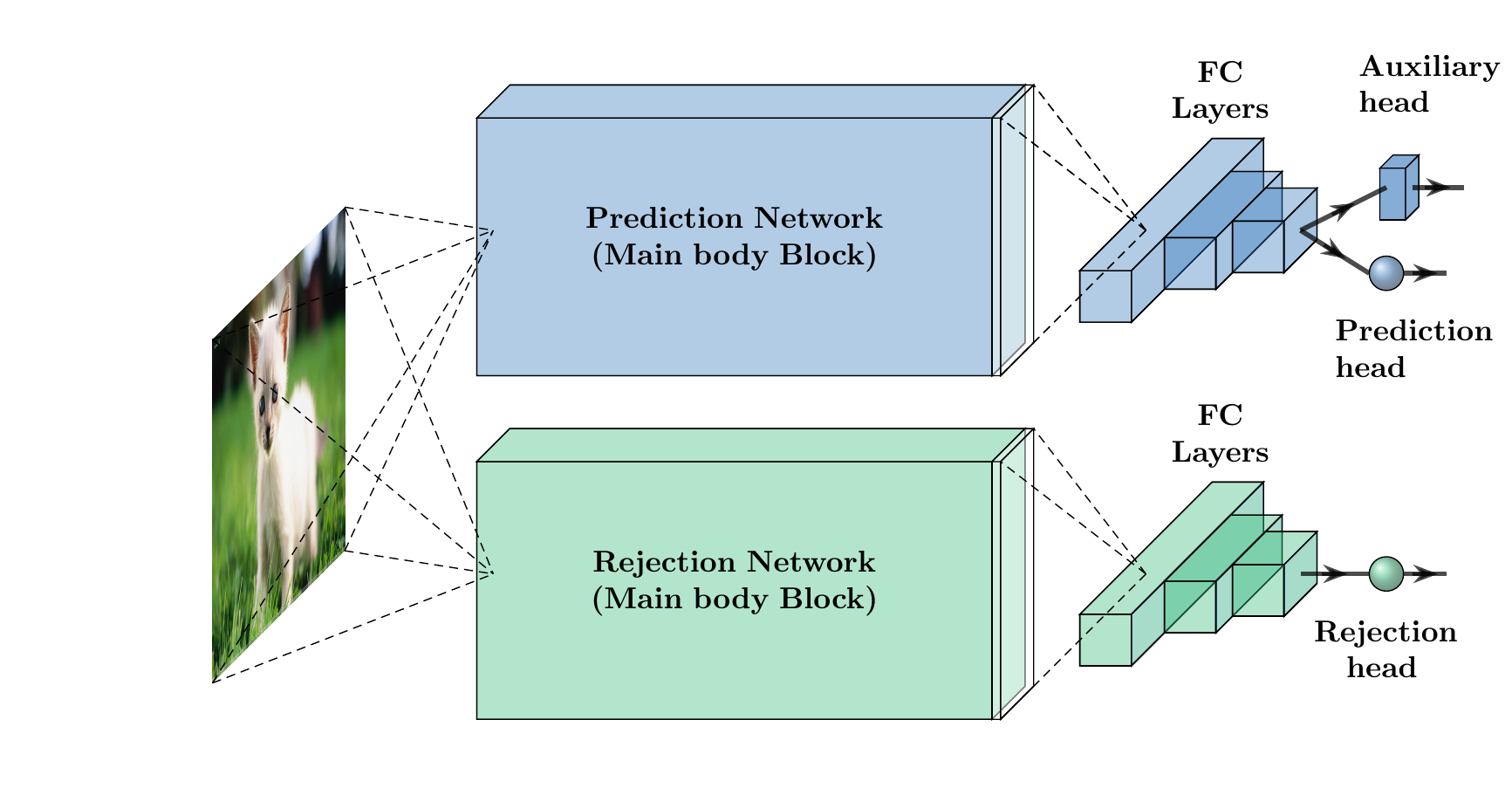}
    \caption{GradCAM Implementation}
    \label{fig:DN-3}
\end{figure}
\subsection{GradCAM Implementation}
\label{sec:GradCAM}
One exciting prospect of deep learning models is their ability to help us understand why a classification decision was made or the important generic features learned during the training process. We followed the GradCAM technique of \cite{selvaraju2017grad} that produces a localization map highlighting important regions in the image corresponding to particular predictions(class) to evaluate our model. We used the architecture described in Fig. \ref{fig:DN-3}, with VGG-16 in both the networks, and executed the GradCAM technique on the sigmoid outputs of the auxiliary head associated with the prediction network. 
 We explored the possibility of having two separate networks, prediction network and rejection network, for each head separately. 
The use of separate networks was adopted to minimize the sharing of features, and subsequently, important features learned by each network could be examined independently.
This helps in visualizing features learned by the prediction network to produce highlighted regions corresponding to the image's different classes.

\end{document}